\documentclass[letterpaper]{article} % DO NOT CHANGE THIS
\usepackage{aaai2027} % DO NOT CHANGE THIS
\nocopyright
\usepackage[hyphens]{url}  % DO NOT CHANGE THIS
\usepackage{graphicx} % DO NOT CHANGE THIS
\usepackage{natbib}  % DO NOT CHANGE THIS AND DO NOT ADD ANY OPTIONS TO IT
\usepackage{caption} % DO NOT CHANGE THIS AND DO NOT ADD ANY OPTIONS TO IT
\usepackage{algorithm}

\usepackage{enumitem}
\usepackage{threeparttable}
\usepackage{color}
\usepackage{colortbl}
\usepackage{xcolor}
\usepackage{multirow}
\usepackage{booktabs}
\usepackage{array}
\usepackage{booktabs}
\usepackage{pifont}   % 提供更美观的叉号

\newcommand{\xmark}{\ding{55}}

\usepackage{amsmath}
\usepackage{amssymb}
\usepackage{amsthm}
\usepackage{mathtools}
\newtheorem{lemma}{Lemma}
\newtheorem{proposition}{Proposition}
\theoremstyle{remark}

\usepackage{algorithm}
\usepackage{algpseudocode}

\usepackage[table]{xcolor} 
\definecolor{iceblue}{RGB}{235, 246, 255}

\usepackage{amsmath} 
\definecolor{growthgreen}{RGB}{0, 160, 0}

\usepackage{newfloat}
\usepackage{listings}
\DeclareCaptionStyle{ruled}{labelfont=normalfont,labelsep=colon,strut=off} % DO NOT CHANGE THIS
\floatstyle{ruled}
\newfloat{listing}{tb}{lst}{}
\floatname{listing}{Listing}

\usepackage{booktabs}

\title{Beyond Shortcut Learning: Hierarchical Spatial Prototype Regularization for Face  Forgery Detection}
\author{
    Chao Shuai\textsuperscript{\rm 1},
    Shaojing Fan\textsuperscript{\rm 2},
    Kun Pan\textsuperscript{\rm 1},
    Jinwu Hu\textsuperscript{\rm 3},
    Weichen Lian\textsuperscript{\rm 1},
    Ziqi Yang\textsuperscript{\rm 1},\\
    Zhongjie Ba\textsuperscript{\rm 1},
    Zhenguang Liu\textsuperscript{\rm 1}\corresponding,
}
\affiliations{
    \textsuperscript{\rm 1}State Key Laboratory of Blockchain and Data Security,
    Zhejiang University, Hangzhou, China\\
    \textsuperscript{\rm 2}National University of Singapore, Singapore\\
    \textsuperscript{\rm 3}South China University of Technology, Guangzhou, China\\
    \{chaoshuai, pankun, yangziqi, zhongjieba\}@zju.edu.cn, kevinbaylor@163.com\\
    liuzhenguang2008@gmail.com, dcsfs@nus.edu.sg, fhujinwu@gmail.com
}

\begin{document}

\maketitle

\begin{abstract}
Recent advances in generative models have enabled highly realistic face forgeries, yet generalizing deepfake detectors to unseen manipulation pipelines remains a major challenge.  One critical yet underexplored issue is that the learned forensic representation itself may fail to remain stable across different manipulation sources. In this paper, we provide new evidence that naively fine-tuned vision models suffer from severe \textit{forensic representation collapse}: their extracted cues are often entangled with unstable, dataset-specific artifacts, resulting in structurally fragile representations and poor cross-domain generalization. Leveraging this insight, we propose \textbf{H}ierarchical \textbf{S}patial \textbf{P}rototype (\textbf{HSP}) regularization, which explicitly models the forgery manifold by learning a structured prototype subspace over spatial features. We further introduce prototype orthogonality and second-order assignment balance, which jointly preserve the geometric independence of latent forgery factors and promote high-effective-rank utilization of the forensic representation, thereby reducing reliance on dataset-specific shortcuts. Empirical results demonstrate that our method substantially improves generalizability, yielding average frame-level AUC gains of 5.6\% and 5.3\% in the cross-dataset and cross-manipulation settings, respectively. Moreover, when used as a plug-in module, HSP boosts the average video-level AUC of existing state-of-the-art detectors to 95.3\% (+1.8\%). These results demonstrate the practical value of HSP as an easy-to-adopt regularization module for enhancing deepfake detectors.
% The code is available at \url{xxx}.
\end{abstract}


\section{Introduction}
\label{sec:introduction}
Deepfake faces are becoming increasingly photorealistic, posing growing risks to digital media security. By 2025, the number of deepfake files shared online had skyrocketed to 8 million, a 16-fold increase from 2023~\cite{deepfakes1}. Against this backdrop, deepfake detection has evolved from handcrafted cues and statistical artifact analysis to deep-learning-based detectors. Recent studies~\cite{fakeradar,fcg,deepshield,vbsat} increasingly adopt large-scale pretrained Vision Transformers~\cite{vit}, \textit{e.g.}, CLIP~\cite{clip}, achieving strong in-domain performance. However, these gains often transfer poorly to unseen manipulation pipelines, where performance degrades substantially. From a representation learning perspective, this limitation arises from shortcut learning: detectors tend to overfit volatile, dataset-specific cues rather than capturing stable forensic patterns.

Motivated by this, we shift the focus from detector design to probing the foundation model itself, and use CLIP, a widely used backbone in deepfake detection, to investigate how forgery evidence is represented. Our preliminary analysis reveals a simple yet overlooked phenomenon: replacing the global \texttt{[CLS]} representation with average pooling over spatial patch tokens improves cross-dataset detection accuracy on Celeb-DF-v2~\cite{celebdf} and DFDCP~\cite{dfdcp}, from 75.4\% and 77.8\% to 76.9\% and 78.6\%, respectively. This improvement motivates the hypothesis that forgery representations should be modeled as a structured feature space rather than compressed into a compact global representation. This hypothesis is consistent with prior work~\cite{effort}, which shows that existing detectors often suffer from low-rank feature collapse, where forgery cues are concentrated along only a few dominant directions.

We further quantitatively analyze the representation structure of a CLIP-initialized ViT-L/14~\cite{vit,clip} fine-tuned on FaceForensics++~\cite{faceforensics++}. The analysis reveals a clear generalization bottleneck: although fine-tuning enables the model to form discriminative forgery-related patterns in the feature space, the learned representation tends to concentrate on a few dominant source-specific artifacts, as detailed in Section~\ref{sec:analysis}. We term this phenomenon \emph{forensic representation collapse}, where detectors capture  a few high-confidence but narrow forgery directions that achieve strong source-domain discrimination but fail to remain stable under unseen manipulations. These observations motivate us to model forgery traces as a \emph{structured manifold} composed of multiple latent forensic factors, as illustrated in Figure~\ref{fig:intro-manifold}(b). In contrast, conventional detectors mainly enforce real--fake separability without explicitly constraining the internal organization of forgery representations, potentially causing heterogeneous forensic cues to remain entangled, as illustrated in Figure~\ref{fig:intro-manifold}(a). We argue that real-world forgeries are intrinsically multi-factorial, where different manipulation processes introduce diverse yet complementary forensic traces. This view is supported by prior forensic studies that capture various manipulation cues, including synthesis artifacts, blending boundaries, geometric inconsistencies, and post-processing traces~\cite{sbi,f3net,ict}.

Based on this intuition, we propose \textbf{Hierarchical Spatial Prototype (HSP) regularization}, a plug-in module that can be seamlessly integrated into both standard baselines and advanced ViT-based deepfake detectors. HSP models learnable prototypes as a structured basis for decomposing token-level features into latent forgery factors. Each spatial token is softly assigned to the prototypes and reconstructed from their weighted combination. Considering that real-world forgeries exhibit diverse and complementary forensic factors, we introduce two theoretically motivated regularization objectives. The \emph{prototype orthogonality} constraint preserves the geometric decoupling of learned forgery factors, preventing multiple prototypes from collapsing into the same dataset-specific dominant directions. The second-order \emph{assignment-balance} objective promotes diverse utilization of the prototype subspace, increasing the effective rank of the forensic representation and preventing token assignments from collapsing onto a limited set of dominant forgery factors. As visualized in Figure~\ref{fig:intro-manifold}(c), the features learned by HSP not only form a clear decision boundary but also exhibit a pronounced multi-factor organization within the forgery representation space. Our contributions can be summarized as follows:

\begin{itemize}
    \item We systematically analyze a critical failure mode in deepfake detection, termed \textit{forensic representation collapse}. Our empirical findings show that naively fine-tuned models are overly dependent on limited and dataset-specific forgery factors and fragile under cross-domain shifts.

     \item We propose \textbf{H}ierarchical \textbf{S}patial \textbf{P}rototype (\textbf{HSP}) regularization, a structured manifold learning framework for forgery detection. HSP introduces learnable spatial prototypes with soft token-to-prototype assignment to decompose forgery evidence into diverse latent forensic factors.

    \item We further introduce two theoretically motivated regularization objectives: prototype orthogonality prevents prototypes from collapsing toward the same source-specific directions, while assignment balance encourages diverse utilization of the prototype subspace and promotes a high-effective-rank forensic representation.
    % \item We further introduce two complementary optimization objectives: prototype orthogonality reduces redundancy among latent forensic factors, while assignment balance distributes feature variation across the prototype-assignment space. Theoretically, they drive the learned forensic representation toward the maximum effective rank permitted by soft assignments, mitigating low-rank collapse and dataset-specific shortcuts to improve cross-domain generalization.
\end{itemize}

\section{Related work}
\label{sec:related-work}
Face forgery detection is often formulated as a binary classification task, where overfitting severely limits generalization to unseen datasets. To improve generalization, early works focus on mining manipulation-specific artifacts, ranging from high-frequency anomalies~\cite{generalizing,leveraging,f3net,freqdebias} to attention-guided fine-grained features~\cite{towards,detection,multi}. Blending artifacts are investigated by Face X-Ray~\cite{face-x-ray} and Chen et al.~\cite{chen}, while~\cite{sbi,prodet,fsgd} synthesize ``pseudo-fakes'' to simulate diverse forgery patterns. Another line of research focuses on consistency analysis:~\cite{uia,dual} identify local spatial inconsistencies,~\cite{ict,id-reveal} exploit identity inconsistencies, and LipForensics~\cite{lips} captures high-level semantic irregularities in temporal mouth movements. Furthermore, some methods leverage diffusion guidance~\cite{dffake} or feature disentanglement~\cite{exposde,ucf,recce,cfm} to capture manipulation-invariant representations.

\begin{figure}[t]
    \centering
    \includegraphics[width=0.48\textwidth]{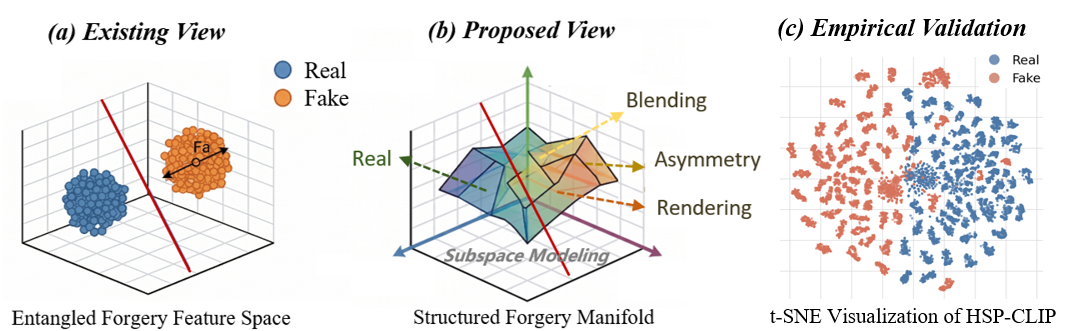}
    \caption{Structured-manifold modeling of forgery representations. (a) Existing methods separate real and fake samples within an entangled feature space. (b) Our approach instead organizes forgery traces into a structured manifold supported by multiple latent forgery factors. (c) The t-SNE visualization provides empirical support for this structured representation.}
    \label{fig:intro-manifold}
\end{figure}

Recent deepfake detection methods adopt large-scale pre-trained foundation models due to their strong representations. Existing approaches can be categorized into three paradigms: 1) Parameter-Efficient Adaptation aligns largely frozen backbones with forensic tasks using lightweight modules. Forensics Adapter~\cite{forensicsadapte} introduces learnable adapters into CLIP, while FCG~\cite{fcg} and plug-and-play methods extend this idea to video detection with efficient spatiotemporal modeling. 2) Feature Disentanglement aims to reduce semantic bias in foundation-model representations. Effort~\cite{effort} separates semantic components from forgery-relevant residuals via SVD, while UDD~\cite{udd} mitigates positional and content bias through transformer token perturbation. 3) Generalization and Explainability target open-set manipulation and interpretable detection. FakeRadar~\cite{fakeradar} models feature-space distributions to distinguish real, known-fake, and unknown-fake samples, while DeepShield~\cite{deepshield} and X2-DFD~\cite{x2dfd} explore local-global analysis and MLLM-assisted explanations, respectively. Our method differs from the existing methods in that it shifts the detection paradigm from learning a monolithic decision boundary to constructing a structured, multi-dimensional forgery subspace. HSP explicitly decomposes complex manipulation traces into independent, orthogonal dimensions, which ensures that the model captures a set of structurally invariant forgery patterns rather than dominant shortcuts.

\section{Empirical Analysis of Generalization Bottlenecks}
\label{sec:analysis}
% A deeper understanding of feature representations learned by deepfake detection models is essential for interpreting their decision mechanisms and improving cross-dataset generalization. To this end, we introduce three forensic-oriented diagnostics that quantify complementary properties of manipulation-related representations. These measures are inspired by clustering evaluation, mutual-information-based dependence analysis, and domain-invariant representation learning, but are tailored here to characterize forgery-specific feature organization.
Understanding how deepfake detectors organize forgery evidence into latent forensic factors is crucial for improving cross-dataset generalization. To this end, we treat unsupervised feature clusters as temporary proxies for latent forensic factors and introduce three diagnostics to characterize the underlying representation structure. Our experiments show that \textit{modeling multiple forgery factors, as instantiated by HSP, provides an effective means of improving the cross-domain generalizability.} A detailed analysis of these diagnostics is provided in Appendix~\ref{app:diagnostic_theory}.

\begin{table}[t]
    \centering
    \renewcommand{\arraystretch}{0.82}
    \caption{Quantitative analysis of feature representations across three models: pretrained and naively finetuned variants of CLIP ViT-L/14, alongside HSP-CLIP.}
    \resizebox{0.9\columnwidth}{!}{
        \begin{tabular}{lcccc}
            \toprule 
            \textbf{Methods} & $\mathcal{\bar H}$ & $\mathcal{I}$ & $\mathcal{P}$ & ACC \\
            \midrule 
            
            \multicolumn{5}{c}{\textbf{FaceForensics++}} \\
            \midrule
            CLIP & $0.553^{\pm 0.009}$ & $0.005^{\pm 0.002}$ & $\mathbf{0.968^{\pm 0.027}}$ & 0.854 \\
            CLIP-FT & $0.997^{\pm 0.001}$ & $0.518^{\pm 0.003}$ & $0.815^{\pm 0.146}$ & 0.931 \\
            \rowcolor{iceblue}
            HSP-CLIP-FT & $\mathbf{0.998^{\pm 0.001}}$ & $\mathbf{0.521^{\pm 0.001}}$ & $0.917^{\pm 0.115}$ & $\mathbf{0.954}$ \\ 
            \midrule
            
            \multicolumn{5}{c}{\textbf{Celeb-DF-v2}} \\
            \midrule
            CLIP & $0.612^{\pm 0.012}$ & $0.016^{\pm 0.003}$ & $0.126^{\pm 0.121}$ & 0.665 \\
            CLIP-FT & $0.996^{\pm 0.001}$ & $0.181^{\pm 0.003}$ & $0.483^{\pm 0.094}$ & 0.786 \\
            \rowcolor{iceblue}
            HSP-CLIP-FT & $\mathbf{0.998^{\pm 0.001}}$ & $\mathbf{0.234^{\pm 0.002}}$ & $\mathbf{0.832^{\pm 0.136}}$ & $\mathbf{0.828}$ \\ 
            \midrule
            
            \multicolumn{5}{c}{\textbf{DFDCP}} \\
            \midrule
            CLIP & $0.683^{\pm 0.058}$ & $0.004^{\pm 0.001}$ & $0.172^{\pm 0.107}$ & 0.547 \\
            CLIP-FT & $0.965^{\pm 0.004}$ & $0.209^{\pm 0.002}$ & $0.544^{\pm 0.136}$ & 0.779 \\
            \rowcolor{iceblue}
            HSP-CLIP-FT & $\mathbf{0.993^{\pm 0.002}}$ & $\mathbf{0.245^{\pm 0.002}}$ & $\mathbf{0.917^{\pm 0.026}}$ & $\mathbf{0.816}$ \\ 
            \midrule

            \multicolumn{5}{c}{\textbf{DFDC}} \\
            \midrule
            CLIP & $0.584^{\pm 0.004}$ & $0.003^{\pm 0.001}$ & $0.476^{\pm 0.147}$ & 0.529 \\
            CLIP-FT & $0.970^{\pm 0.004}$ & $0.175^{\pm 0.001}$ & $0.523^{\pm 0.219}$ & 0.725 \\
            \rowcolor{iceblue}
            HSP-CLIP-FT & $\mathbf{0.993^{\pm 0.003}}$ & $\mathbf{0.199^{\pm 0.001}}$ & $\mathbf{0.817^{\pm 0.096}}$ & $\mathbf{0.775}$ \\ 
            \bottomrule
        \end{tabular}
    }
    \label{tab:analysis-metrics}
\end{table}
\subsection{Feature-Centric Perspective Metrics}
We denote the dataset as $\mathcal{D}=\{(x_i,y_i)\}_{i=1}^N$, where $x_i\in X$ represents the input image space and $y_i\in Y=\{0,1\}$ denotes the binary label. Given an input image $x\in\mathbb{R}^{H\times W\times C}$, we extract non-overlapping patch embeddings $Z=f_\theta(x)\in\mathbb{R}^{L\times D}$, where $L$ and $D$ denote the number of patches and feature dimension, respectively.

\textbf{Cluster-level Fake Density $\mathcal{H}$.}
For patch-level features $z_{i,j}$, we apply $K$-means with $N_c=64$ clusters and compute the fake density of cluster $n_c$ as:

\begin{equation}
\mathcal{H}(n_c)=\frac{1}{|\mathcal{S}_{n_c}|}
\sum_{i=1}^{N}\sum_{j=1}^{L}
\mathbb{I}(y_i=1\land(i,j)\in\mathcal{S}_{n_c}),
\end{equation}
where $\mathcal{S}_{n_c}=\{(i,j)\mid \text{assign}(z_{i,j})=n_c\}$ denotes the patches assigned to cluster $n_c$, and $\mathbb{I}(\cdot)$ is the indicator function. Intuitively, $\mathcal{H}(n_c)$ measures the probability that a patch in cluster $n_c$ comes from a forged image. A higher value indicates that the cluster captures more forgery-related evidence while separating it from genuine visual patterns.

\textbf{Cluster--Label Dependence $\mathcal{I}$.}
We employ mutual information~\cite{mi} between cluster assignments and ground-truth labels to quantify the discriminative structure of clusters. Let $C$ denote the cluster assignment variable, where $n_c\in\{1,\dots,N_c\}$ indexes a cluster, and $Y\in\{0,1\}$ denotes the label. We define $\mathcal{I}$ as
\begin{equation}
\mathcal{I}(C;Y)
=
\sum_{n_c=1}^{N_c}\sum_{y\in\{0,1\}}
p(n_c)p(y\mid n_c)
\log \frac{p(y\mid n_c)}{p(y)} .
\end{equation}

\begin{equation}
p(y\mid n_c)
=
\frac{
\left|\{(i,j)\mid(i,j)\in S_{n_c}\land y_i=y\}\right|
}{
|S_{n_c}|
}.
\end{equation}
Here, $p(y\mid n_c)$ is the empirical probability that a patch assigned to cluster $n_c$ belongs to class $y$, while $p(n_c)$ and $p(y)$ denote the corresponding marginal probabilities. A higher $\mathcal{I}$ indicates stronger alignment between the cluster structure and forensic labels, suggesting better separation between real and manipulated samples.
% \textbf{Cluster–Label Dependence $\mathcal{I}$. } To quantify the discriminative power of the unsupervised clusters, we employ Mutual Information~\cite{mi} between the cluster assignments and the ground-truth labels. Let $C$ be the random variable representing the cluster assignment $n_c \in \{1, \dots, N_c\}$, and we define $\mathcal{I}$ as:

% \begin{equation}
% I(C;Y)=\sum_{n_c=1}^{N_c}\sum_{y\in\{0,1\}} p(n_c,y)\,
% \log \frac{p(n_c,y)}{p(n_c)\,p(y)} ,
% \end{equation}

% \begin{equation}
% p(n_c,y)=\frac{\left|\left\{(i,j)\mid (i,j)\in S_{n_c}\ \wedge\ y_i=y\right\}\right|}
% {\sum_{n'=1}^{N_c}\left|S_{n'}\right|} .
% \end{equation}
% where $p(n_k, y)$ is the joint empirical probability of observing a patch in cluster $n_k$ with label $y$, and $p(n_c)$ and $p(y)$ are the marginal probabilities. A higher $\mathcal{I}$ indicates that the clusters act as effective proxies for forensic labels, separating real and manipulated patches in the feature space.

\textbf{Cross-Dataset Fake Density Correlation $\mathcal{P}$.}
$\mathcal{P}$ characterizes whether dominant fake clusters preserve their forensic relevance across domains. For a dataset $\mathcal{D}$, we define the fake density vector:
\begin{equation}
\mathbf{V} = \left[ \mathcal{H}(1), \mathcal{H}(2), \dots, \mathcal{H}(N_c) \right]^\top.
\end{equation}
Let $\mathbf{V}_{tr}$ and $\mathbf{V}_{ts}$ denote the vectors computed on the training and test datasets using cluster centroids fixed from the training set. We select the top-$M$ fake-dominated clusters in $\mathbf{V}_{tr}$ as an index set $\mathcal{M}$, where $|\mathcal{M}|=M$, and compute the Pearson correlation coefficient:

\begin{equation}
P =
\frac{
\sum_{m \in \mathcal{M}} (V_{\mathrm{tr}}[m]-\bar{V}_{\mathrm{tr}})(V_{\mathrm{ts}}[m]-\bar{V}_{\mathrm{ts}})
}{
\sqrt{\sum_{m \in \mathcal{M}} (V_{\mathrm{tr}}[m]-\bar{V}_{\mathrm{tr}})^2}
\sqrt{\sum_{m \in \mathcal{M}} (V_{\mathrm{ts}}[m]-\bar{V}_{\mathrm{ts}})^2}
},
\end{equation}
where $\bar{V}$ denotes the mean of the selected $M$ elements. A higher $P$ indicates that fake-dominated clusters in the training domain remain consistent in the test domain, suggesting reliance on transferable forgery patterns.

\begin{figure}
    \centering
    \includegraphics[width=0.9\linewidth]{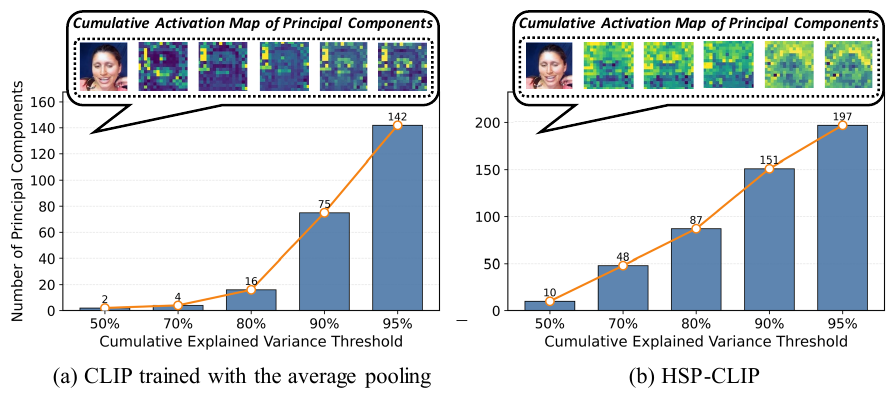}
    \caption{PCA-based analysis of forgery representations. HSP-CLIP requires more principal components to explain the same variance and produces broader cumulative spatial activations, suggesting that it captures more latent forensic factors within a structured subspace.}
    % Compared with naively fine-tuned CLIP, HSP-CLIP consistently requires substantially more principal components to explain the same amount of variance, indicating that forgery-relevant information is distributed across a richer and structured subspace rather than concentrated in a few dominant directions. The accompanying cumulative activation maps of principal components further reveal that naively fine-tuned CLIP accumulates its responses on only a few patches, whereas HSP-CLIP exhibits broader and more complementary spatial regions.}
    \label{fig:analysis-pca}
\end{figure}

\subsection{Empirical Analysis on Deepfake Benchmarks}
 \textbf{Key Observation: }\textit{Naive fine-tuning improves source-domain separability but induces a shortcut-dominated, low-rank forensic representation, whereas HSP improves cross-domain feature stability by modeling multiple forgery factors.}

We apply the diagnostics to analyze state-of-the-art detectors trained on FaceForensics++~\cite{faceforensics++} and evaluated on Celeb-DF-v2~\cite{celebdf}, DFDCP~\cite{dfdcp}, and DFDC~\cite{dfdc}. Our analysis centers on ViT-L/14 initialized with CLIP~\cite{clip} weights. For each target dataset, we average results over five trials with 2,000 real and 2,000 fake images, using 64 cluster centroids from the training set and computing $\bar{\mathcal{H}}$ and $\mathcal{P}$ over the top-10 fake-dominated clusters. 

Results are summarized in Table~\ref{tab:analysis-metrics}, with additional results for different values of $M$ provided in Appendix~\ref{a-c}. We observe that: 1) vanilla backbones (\textit{e.g.}, CLIP) yield nearly random scores (\textit{e.g.}, $\bar{\mathcal{H}}\approx0.55$, $\mathcal{I}\approx0.001$), indicating that generic visual pre-training lacks intrinsic sensitivity to forgery-specific evidence; 2) naive fine-tuning (CLIP-FT) achieves higher $\bar{\mathcal{H}}$, suggesting that it captures forgery-related features and improves source-domain discrimination. However, $\mathcal{I}$ and $\mathcal{P}$ substantially degrade under OOD evaluation, revealing that the learned patterns are dominated by source-specific factors and lack stable cross-domain forensic structures; 3) in contrast, HSP-CLIP preserves high cluster-level fake density while consistently improving both $\mathcal{I}$ and $\mathcal{P}$, indicating that our multi-factor forgery representation learning organizes forgery evidence into a more domain-stable representation and improves cross-domain generalization.

We further perform PCA on patch embeddings extracted from 5,000 test images and compare the number of principal components required to reach different explained-variance thresholds. We further visualize the cumulative activation maps of the leading principal components by accumulating their token-wise projection energy. Intuitively, if forgery evidence collapses into a few shortcut-dominated directions, the representation should be highly low-rank, requiring only a small number of principal components while producing spatial responses concentrated on limited patches. Conversely, a structured forgery manifold should require more principal components and exhibit broader spatial activation across complementary forensic regions. As shown in Figure~\ref{fig:analysis-pca}, naively fine-tuned CLIP explains most of its variance with only a few principal directions, whereas HSP-CLIP requires substantially more principal components to reach the same explained-variance thresholds. The cumulative activation maps further reveal a clear geometric difference: CLIP mainly responds to a small number of localized patches, while HSP-CLIP distributes its responses across broader and more complementary spatial regions. These results suggest that HSP-CLIP organizes forgery evidence into a structured manifold rather than compressing it into a shortcut-dominated low-rank representation, providing empirical support for the structured forgery manifold proposed in Figure~\ref{fig:intro-manifold}.

\section{Methodology}
\label{sec:methodology}
% Our analysis in Section~\ref{sec:analysis} reveals that the primary obstacle to cross-dataset deepfake detection is the lack of a stable structure for organizing forgery evidence. Under conventional fine-tuning, intermediate token features are only indirectly supervised by the final binary classification objective. As a result, they are free to drift toward arbitrary discriminative directions associated with compression patterns or manipulation-specific artifacts. Although these directions may produce high accuracy on the training distribution, they do not necessarily correspond to transferable forgery features. We therefore formulate cross-dataset deepfake detection from the perspective of \emph{structured forgery manifold learning} and propose a plug-in module termed \textbf{Hierarchical Spatial Prototype} (HSP), as illustrated in Figure~\ref{fig:method-overview}. The core idea is to constrain intermediate spatial representations to be explained by a finite set of shared and complementary forgery factors.
% Our analysis in Section~\ref{sec:analysis} reveals that the primary obstacle to cross-dataset deepfake detection is the lack of a stable structure for organizing forgery evidence. We therefore formulate cross-dataset deepfake detection from the perspective of \emph{structured forgery manifold learning} and propose a plug-in module termed \textbf{Hierarchical Spatial Prototype} (HSP), as illustrated in Figure~\ref{fig:method-overview}. The core idea is to constrain intermediate spatial representations to be explained by a set of complementary forgery factors. 

The observations in Section~\ref{sec:analysis} suggest that weak cross-domain performance stems from narrow, low-rank representations dominated by a few source-specific forgery cues. Accordingly, we formulate deepfake detection from the perspective of \emph{structured forgery manifold learning} and propose a plug-in module termed \textbf{Hierarchical Spatial Prototype} (HSP), as illustrated in Figure~\ref{fig:method-overview}. The core idea is to constrain intermediate spatial representations to be explained by a set of complementary forgery factors.  

\subsection{Prototype-Supported Forgery Manifold}
\label{sec:prototype_reconstruction}
\textit{How should we organize heterogeneous forgery evidence into a structured feature manifold that can accommodate multiple latent forgery factors?}
To address this question, given intermediate token embeddings $Z \in \mathbb{R}^{B \times N \times D}$ from a ViT layer, we introduce $K$ learnable prototypes $P=\{p_k\}_{k=1}^{K}$, where each prototype represents a latent spatial forgery factor. Since real-world manipulations typically involve multiple intertwined effects, such as blending, geometric inconsistency, and post-processing, assigning each token to a single prototype is overly restrictive. We therefore adopt temperature-scaled soft assignment, allowing each token $z_i\in\mathbb{R}^{D}$ to be reconstructed as a weighted combination of all prototypes:
\begin{equation}
\hat{z}_i
=
\sum_{k=1}^{K}
\underbrace{
\frac{
\exp\!\left(\cos(z_i,p_k)/\tau\right)
}{
\sum_{j=1}^{K}
\exp\!\left(\cos(z_i,p_j)/\tau\right)
}
}_{A_{i,k}}
\, p_k ,
\end{equation}
where $\tau$ is the temperature parameter and $A_{i,k}$ denotes the soft assignment weight of token $z_i$ to prototype $p_k$.

% Such a soft reconstruction mechanism is therefore more consistent with the structured-manifold perspective introduced in Section~\ref{sec:analysis}.
HSP is designed as a plug-in module that preserves the original representational capacity while imposing structured organization on intermediate features. We introduce prototype reconstruction as an auxiliary constraint rather than replacing original token representations. Direct replacement would force all tokens into a finite prototype-supported subspace and may discard subtle forensic cues beyond the learned prototypes. Instead, the reconstruction objective softly aligns features with shared forgery factors while retaining the flexibility of pretrained representations. Specifically, we adopt a dual-path reconstruction loss with stop-gradient:
\begin{equation}
\mathcal{L}_{\mathrm{rec}}
=
\underbrace{\| Z - \mathrm{sg}(\hat{Z}) \|_2^2}_{\shortstack{\scriptsize \text{feature-to-prototype} \\ \scriptsize \text{alignment term}}}
+ \underbrace{\alpha \| \mathrm{sg}(Z) - \hat{Z} \|_2^2}_{\shortstack{\scriptsize \text{prototype-to-feature} \\ \scriptsize \text{reconstruction term}}},
\end{equation}
where $\mathrm{sg}(\cdot)$ denotes the stop-gradient operator, and the weighting coefficient $\alpha$ is set to 2 by default. Similar designs have been used to stabilize representation learning~\cite{vqvae,simsiam}.
% By constraining token features to be expressed through a compact set of representative prototypes, our method encourages heterogeneous forgery patterns to be absorbed into purified latent clusters, which helps explain the higher Cluster-level Fake Density ($\bar{\mathcal{H}}$) observed in Table~\ref{tab:analysis-metrics}.

\begin{figure*}[t]
    \centering
    \includegraphics[width=0.8\linewidth]{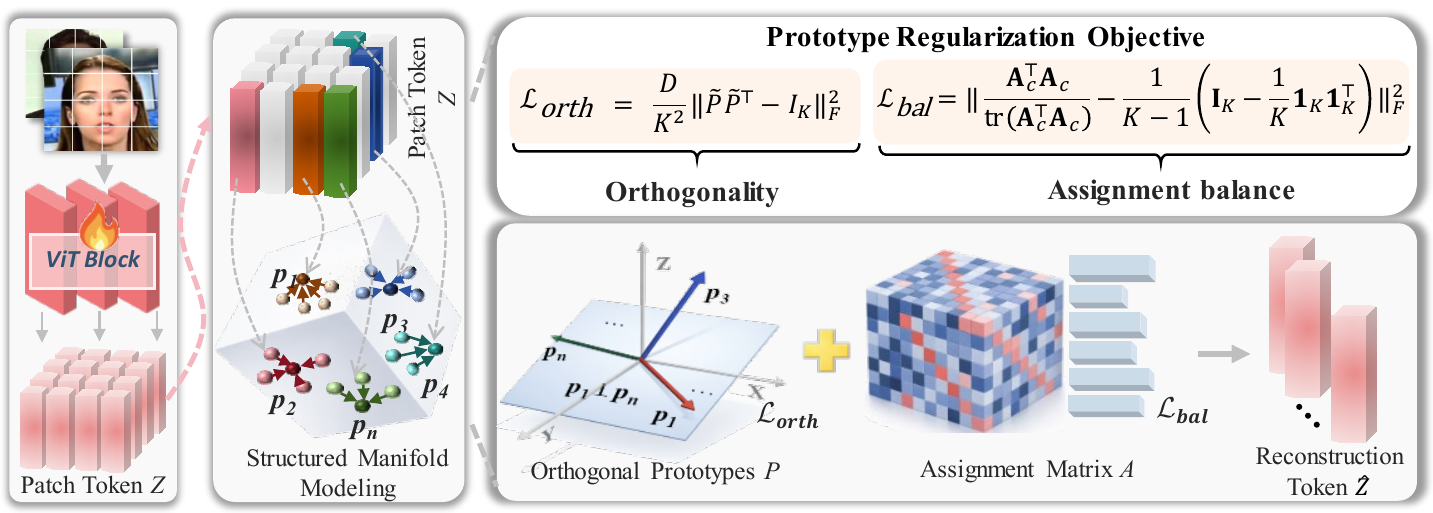}
    \caption{Overview of the proposed \textbf{H}ierarchical \textbf{S}patial \textbf{P}rototype (\textbf{HSP}) module for structured-manifold modeling. HSP introduces $K$ learnable spatial prototypes to model latent forgery factors. A soft assignment matrix $A$ decomposes each token into a combination of multiple prototypes and enables heterogeneous forgery traces to be modeled within a prototype-supported manifold. Prototype orthogonality keeps these forgery factors geometrically decoupled, while second-order assignment balance encourages diverse utilization of the prototype space and prevents the representation from collapsing into a few dominant factors.}
    % Together, these constraints provably drive the learned forensic features toward their maximum admissible effective rank, preventing low-rank collapse and reliance on source-specific shortcuts.
    \label{fig:method-overview}
\end{figure*}

\subsection{Structured Forgery Manifold Regularization}
Based on the learned forgery prototypes, we constrain both prototype geometry and assignment distribution to ensure complementary factor modeling and effective prototype utilization. We first impose an \textit{orthogonality} regularization on the normalized prototype matrix:
\begin{equation}
\mathcal{L}_{\mathrm{orth}}
=
\frac{D}{K^2}
\left\|\tilde{P}\tilde{P}^{\top}-I_K\right\|_F^2
\;\equiv\;
\frac{D}{K^2}
\sum_{i\ne j}
\left(
\frac{p_i^{\top}p_j}{\|p_i\|_2\,\|p_j\|_2}
\right)^2,
\end{equation}
where $\tilde{P}\in\mathbb{R}^{K\times D}$ is the row-wise normalized prototype matrix, $I_K$ is the identity matrix, and $\|\cdot\|_F$ denotes the Frobenius norm. Although orthogonality encourages diverse forgery factors, assignment collapse may still concentrate features on a few dominant prototypes, leaving the representation effectively low-rank. To this end, we model the second-order geometry of the assignment space. Let $\mathbf A\in\mathbb{R}^{BN\times K}$ denote the assignment matrix of all tokens in a mini-batch. Since each row of $\mathbf A$ is a probability vector, $\mathbf A\mathbf 1_K=\mathbf 1_{BN}$. We center the assignment matrix and introduce the \textit{assignment-balance} regularizer:
\begin{equation}
\begin{aligned}
\mathcal L_{\mathrm{bal}}
&=
\left\|
\frac{
\mathbf A_c^{\top}\mathbf A_c
}{
\operatorname{tr}(\mathbf A_c^{\top}\mathbf A_c)
}
-
\frac{1}{K-1}
\left(
\mathbf I_K-\frac{1}{K}\mathbf 1_K\mathbf 1_K^{\top}
\right)
\right\|_F^2, \\
\bar{\mathbf a}
&=
\frac{1}{BN}\mathbf A^{\top}\mathbf 1_{BN},
\qquad
\mathbf A_c
=
\mathbf A-\mathbf 1_{BN}\bar{\mathbf a}^{\top}.
\end{aligned}
\end{equation}
The matrix $\mathbf I_K-\frac{1}{K}\mathbf 1_K\mathbf 1_K^{\top}$ is the rank-$(K-1)$ projector onto the centered assignment subspace. Therefore, $\mathcal L_{\mathrm{bal}}$ encourages the forgery manifold to exploit all admissible prototype-assignment directions instead of collapsing onto a few source-specific prototypes.

\subsection{Effective Rank of the Forensic Representation}
Prototype orthogonality preserves non-redundant forgery directions, while assignment balance distributes feature variation over the full subspace permitted by soft assignments; together, they drive the forensic
representation toward its maximum admissible effective rank.

% maximize the effective rank of the learned deepfake forensic representation.

\paragraph{Proposition.}
Let $\widetilde{\mathbf P}\in\mathbb{R}^{K\times D}$ denote the row-normalized prototype matrix, and let $\widehat{\mathbf Z}_c=\mathbf A_c\widetilde{\mathbf P}$ be the centered reconstructed forensic features. Assume $K\leq D$. If $\mathcal L_{\mathrm{orth}}\rightarrow0$ and $\mathcal L_{\mathrm{bal}}\rightarrow0$, then the trace-normalized Gram matrix of $\widehat{\mathbf Z}_c$ converges to a rank-$(K-1)$ spectrum with equal nonzero eigenvalues. Consequently, its effective rank converges to the maximum admissible value, i.e., $r_{\mathrm{eff}}(\widehat{\mathbf Z}_c^{\top}\widehat{\mathbf Z}_c)\rightarrow K-1$.

\paragraph{Proof.}
Because each row of $\mathbf A$ is a probability vector, $\mathbf A\mathbf 1_K=\mathbf 1_{BN}$. Moreover, from $\bar{\mathbf a}=\frac{1}{BN}\mathbf A^{\top}\mathbf 1_{BN}$, we have $\bar{\mathbf a}^{\top}\mathbf 1_K=1$. Therefore,
\begin{equation}
\mathbf A_c\mathbf 1_K
=
\left(
\mathbf A-\mathbf 1_{BN}\bar{\mathbf a}^{\top}
\right)\mathbf 1_K
=
\mathbf 1_{BN}
-
\mathbf 1_{BN}
\left(
\bar{\mathbf a}^{\top}\mathbf 1_K
\right)
=
\mathbf 0,
\label{eq:centered_assignment_nullspace}
\end{equation}
which implies $\operatorname{rank}(\mathbf A_c)\leq K-1$.

As $\mathcal L_{\mathrm{bal}}\rightarrow0$, the trace-normalized assignment covariance converges to the normalized projector onto the centered assignment subspace:
\begin{equation}
\frac{
\mathbf A_c^{\top}\mathbf A_c
}{
\operatorname{tr}\!\left(
\mathbf A_c^{\top}\mathbf A_c
\right)
}
\rightarrow
\frac{1}{K-1}
\left(
\mathbf I_K
-
\frac{1}{K}
\mathbf 1_K\mathbf 1_K^{\top}
\right).
\label{eq:assignment_covariance_limit}
\end{equation}
The matrix on the right has one zero eigenvalue along $\mathbf 1_K$ and $K-1$ equal nonzero eigenvalues, each equal to $1/(K-1)$. Meanwhile, $\mathcal L_{\mathrm{orth}}\rightarrow0$ implies $\widetilde{\mathbf P}\widetilde{\mathbf P}^{\top}\rightarrow\mathbf I_K$, so the rows of $\widetilde{\mathbf P}$ become orthonormal. The Gram matrix of the centered reconstructed features is
$\widehat{\mathbf Z}_c^{\top}\widehat{\mathbf Z}_c
=
\widetilde{\mathbf P}^{\top}
\mathbf A_c^{\top}\mathbf A_c
\widetilde{\mathbf P}.$

Because $\widetilde{\mathbf P}$ has orthonormal rows, this transformation preserves the nonzero eigenvalues of $\mathbf A_c^{\top}\mathbf A_c$. Therefore, the trace-normalized Gram matrix of $\widehat{\mathbf Z}_c$ also has $K-1$ equal nonzero eigenvalues, each converging to $1/(K-1)$. By the definition of effective rank,

\begin{equation}
\begin{aligned}
r_{\mathrm{eff}}\!\left(
\widehat{\mathbf Z}_c^{\top}\widehat{\mathbf Z}_c
\right)
&\rightarrow
\exp\!\left(
-\sum_{k=1}^{K-1}
\frac{1}{K-1}
\log\frac{1}{K-1}
\right) \\
&= K-1.
\end{aligned}
\end{equation}
\hfill$\square$

The reconstruction objective encourages
$\mathbf Z_c$ to remain close to its prototype-based reconstruction
$\widehat{\mathbf Z}_c$. Let
$\mathbf E_c=\mathbf Z_c-\widehat{\mathbf Z}_c$ denote the reconstruction error. We have
\begin{equation}
\begin{aligned}
\left\|
\mathbf Z_c^{\top}\mathbf Z_c
-
\widehat{\mathbf Z}_c^{\top}\widehat{\mathbf Z}_c
\right\|_2
&\leq
2\left\|\widehat{\mathbf Z}_c\right\|_2
\left\|\mathbf E_c\right\|_2
+
\left\|\mathbf E_c\right\|_2^2 .
\end{aligned}
\label{eq:gram_perturbation}
\end{equation}
By Weyl's eigenvalue perturbation inequality~\cite{horn2012matrix}, the eigenvalue deviation between $\mathbf Z_c^{\top}\mathbf Z_c$ and $\widehat{\mathbf Z}_c^{\top}\widehat{\mathbf Z}_c$ is bounded by Eq.~\eqref{eq:gram_perturbation}. Therefore, with small reconstruction error, the Gram spectra remain close, allowing the high-effective-rank structure of $\widehat{\mathbf Z}_c$ to be approximately inherited by $\mathbf Z_c$.

% Intuitively, $\mathcal{L}_{\mathrm{orth}}$ prevents multiple prototypes from collapsing onto the same dataset-specific artifact direction, thereby preserving a geometrically diverse prototype basis. $\mathcal{L}_{\mathrm{bal}}$ discourages assignment collapse onto a few ``easy'' prototypes and promotes broader coverage of fine-grained forgery cues. Their combination thus addresses both prototype redundancy and usage imbalance, transforming the intermediate ViT representation from a shortcut-dominated space into a more structured and transferable forgery space. This directly corresponds to our analysis of $\mathcal{H}$, $\mathcal{I}$, and $\mathcal{P}$ in Table~\ref{tab:analysis-metrics}: broader prototype coverage strengthens feature separability ($\mathcal{I}$), while a non-redundant prototype basis improves cross-dataset structural consistency ($\mathcal{P}$).

\begin{table*}[t]
\centering
\renewcommand{\arraystretch}{0.7}
\caption{Comparison of video- and frame-level AUC (\%) on cross-dataset benchmarks. $\dagger$ denotes results from official models or our reproduction; other results are from~\cite{effort,x2dfd} and original papers.}
\label{tab:exp-cross-dataset}
\setlength{\tabcolsep}{2.5pt}
\resizebox{0.9\linewidth}{!}{%
\begin{tabular}{c|cccccc|c|cccccc}
\toprule
\multirow{2}{*}{\textbf{Methods}} 
& \multicolumn{6}{c|}{\textbf{Video-level AUC}} 
& \multirow{2}{*}{\textbf{Methods}} 
& \multicolumn{6}{c}{\textbf{Frame-level AUC}} \\
\cmidrule(lr){2-7} \cmidrule(lr){9-14}
& CDF-v2 & CDF-v1 & DFDC & DFDCP & DFD & \textbf{Avg.}
&
& CDF-v2 & CDF-v1 & DFDC & DFDCP & DFD & \textbf{Avg.} \\
\midrule \midrule

F3Net (ECCV'20) & 78.9 & -- & 71.8 & 74.9 & 84.4 & 77.5
& Efficient-b4$\dagger$ (ICML'19) & 74.9 & 70.5 & 69.6 & 72.8 & 81.5 & 73.9 \\

LaV (MM'23) & 92.2 & 84.7 & 82.0 & 89.7 & 96.8 & 89.1
& Face X-Ray (CVPR'20) & 67.9 & 70.9 & 63.3 & 69.4 & 76.7 & 70.0 \\

IID (CVPR'23) & 83.8 & -- & 70.0 & 68.9 & 93.9 & 79.2
& RECCE (CVPR'22) & 73.2 & -- & 71.3 & 74.2 & 74.2 & 73.2 \\

CFM (TIFS'23) & 89.7 & -- & 70.6 & 80.2 & 95.2 & 83.9
& UCF$\dagger$ (ICCV'23) & 82.4 & 77.9 & 80.5 & 75.9 & 87.8 & 80.9 \\

SBIs (CVPR'22) & 90.6 & -- & 75.2 & 87.7 & 88.2 & 85.4
& ICT (CVPR'23) & 85.7 & 81.4 & -- & -- & 84.1 & 83.7 \\

ProDe$\dagger$ (NeurIPS'24) & 92.6 & 95.7 & 70.7 & 82.8 & 95.4 & 87.4
& ProDe$\dagger$ (NeurIPS'24) & 84.5 & 90.9 & 72.4 & 81.2 & -- & 82.3 \\

CDFA$\dagger$ (ECCV'24) & 93.8 & 92.1 & 83.0 & 88.1 & 95.4 & 90.5
& DFFake (NeurIPS'24) & 80.5 & -- & -- & 81.0 & 90.4 & 84.0 \\

FakeRadar (ICCV'25) & 91.7 & -- & 84.1 & 88.5 & 96.2 & 90.1
& ExposDe$\dagger$ (AAAI'24) & 86.4 & 81.8 & 72.1 & 85.1 & 89.8 & 83.0 \\

VbSaT (CVPR'25) & 94.7 & -- & 84.3 & 90.9 & 96.5 & 91.6
& DAID (NeurIPS'25) & 84.4 & -- & 66.9 & -- & 91.2 & 80.8 \\

DeepShield (ICCV'25) & 92.2 & -- & 82.8 & \textbf{93.2} & 96.1 & 91.1
& UDD (AAAI'25) & 86.9 & -- & 75.8 & 85.6 & 91.0 & 84.8 \\

FCG (CVPR'25) & 95.0 & -- & 81.8 & -- & -- & 88.4
& FSgD (NeurIPS'25) & 86.7 & 90.1 & -- & 81.8 & 82.1 & 85.2 \\

$\chi^2$-DFD (NeurIPS'25) & 95.5 & -- & 85.3 & 91.2 & 95.7 & 91.9
& LSDA (CVPR'24) & 83.0 & 86.7 & 73.6 & 81.5 & 88.0 & 82.6 \\

Effort$\dagger$ (ICML'25) & 95.6 & 91.8 & 84.3 & 90.9 & 96.5 & 91.8
& FreqDebias (CVPR'25) & 83.6 & 87.5 & 74.1 & 82.4 & 86.8 & 82.9 \\
\midrule

\rowcolor{iceblue}
HSP-CLIP
& \textbf{96.4} \textsuperscript{\textcolor{growthgreen}{\textbf{+0.8}}}
& \textbf{97.1} \textsuperscript{\textcolor{growthgreen}{\textbf{+1.4}}}
& \textbf{89.2} \textsuperscript{\textcolor{growthgreen}{\textbf{+3.9}}}
& 92.2
& \textbf{97.5} \textsuperscript{\textcolor{growthgreen}{\textbf{+0.7}}}
& \textbf{94.5} \textsuperscript{\textcolor{growthgreen}{\textbf{+2.6}}}
& HSP-CLIP
& \textbf{90.9} \textsuperscript{\textcolor{growthgreen}{\textbf{+4.0}}}
& \textbf{91.9} \textsuperscript{\textcolor{growthgreen}{\textbf{+1.0}}}
& \textbf{86.7} \textsuperscript{\textcolor{growthgreen}{\textbf{+6.2}}}
& \textbf{90.0} \textsuperscript{\textcolor{growthgreen}{\textbf{+4.4}}}
& \textbf{94.3} \textsuperscript{\textcolor{growthgreen}{\textbf{+3.1}}}
& \textbf{90.8} \textsuperscript{\textcolor{growthgreen}{\textbf{+5.6}}} \\
\bottomrule
\end{tabular}%
}
\end{table*}

\begin{table*}[t]
\centering
\renewcommand{\arraystretch}{0.7}
\caption{Video-level and frame-level AUC (\%) comparison on unseen manipulation methods. Results are reported on six representative face-swapping methods selected from DF40, where all samples are generated based on the FF++ domain.}
\label{tab:exp-cross-manipulation}
\setlength{\tabcolsep}{0.2pt}
\resizebox{0.9\linewidth}{!}{%
\begin{tabular}{c|ccccccc|c|ccccccc}
\toprule
\multirow{2}{*}{\textbf{Methods}} 
& \multicolumn{7}{c|}{\textbf{Video-level AUC}} 
& \multirow{2}{*}{\textbf{Methods}} 
& \multicolumn{7}{c}{\textbf{Frame-level AUC}} \\
\cmidrule(lr){2-8} \cmidrule(lr){10-16}
& UniFace & e4s & FaceDancer & FSGAN & InSwap & SimSwap & \textbf{Avg.}
& 
& UniFace & e4s & FaceDancer & FSGAN & InSwap & SimSwap & \textbf{Avg.} \\ 
\midrule \midrule
RECCE (CVPR'22) & 89.8 & 68.3 & 84.8 & 94.9 & 84.8 & 76.8 & 83.2 & RECCE (CVPR'22) & 84.2 & 65.2 & 78.3 & 88.4 & 79.5 & 73.0 & 78.1 \\
SBIs (CVPR'22) & 72.4 & 75.0 & 59.4 & 80.3 & 71.2 & 70.1 & 71.4 & IID (CVPR'23) & 79.5 & 71.0 & 79.0 & 86.4 & 74.4 & 64.0 & 75.7 \\
UCF (ICCV'23) & 83.1 & 73.1 & 86.2 & 93.7 & 80.9 & 64.7 & 80.3 & SBIs (CVPR'22) & 64.4 & 69.0 & 44.7 & 87.9 & 63.3 & 56.8 & 64.4 \\
IID (CVPR'23) & 83.9 & 76.6 & 84.4 & 92.7 & 78.9 & 64.4 & 80.2 & UCF (ICCV'23) & 78.7 & 69.2 & 80.0 & 88.1 & 76.8 & 64.9 & 77.5 \\
LSDA (CVPR'24) & 87.2 & 69.4 & 72.1 & 93.9 & 85.5 & 69.3 & 79.6 & LSDA (CVPR'24) & 85.4 & 68.4 & 75.9 & 83.2 & 81.0 & 72.7 & 77.8 \\
ProDe (NeurIPS'24) & 90.8 & 77.1 & 74.7 & 92.8 & 83.7 & 84.4 & 83.9 & CDFA (ECCV'24) & 76.5 & 67.4 & 75.4 & 84.8 & 72.0 & 76.1 & 75.9 \\
CDFA (ECCV'24) & 76.2 & 63.1 & 80.3 & 94.2 & 77.2 & 75.7 & 77.8 & ProDe (NeurIPS'24) & 84.5 & 71.0 & 73.6 & 86.5 & 78.8 & 77.8 & 78.7 \\
VbSaT (CVPR'25) & 96.0 & 98.0 & 91.6 & 96.4 & 93.7 & 93.1 & 94.8 & FSgD (NeurIPS'25) & 91.8 & 87.5 & 83.0 & 86.3 & 87.4 & \textbf{91.0} & 87.8 \\
Effort (ICML'25) & 96.2 & 98.3 & 92.6 & 95.7 & 93.6 & 92.6 & 94.8 & $\chi^2\text{-DFD}$ (NeurIPS'25) & 85.2 & 91.2 & 83.8 & 89.9 & 78.4 & 84.9 & 85.6 \\ \midrule
\rowcolor{iceblue}
HSP-CLIP
& \textbf{96.6} \textsuperscript{\textcolor{growthgreen}{\textbf{+0.4}}}
& \textbf{99.0} \textsuperscript{\textcolor{growthgreen}{\textbf{+0.7}}}
& \textbf{96.2} \textsuperscript{\textcolor{growthgreen}{\textbf{+3.6}}}
& \textbf{98.6} \textsuperscript{\textcolor{growthgreen}{\textbf{+2.2}}}
& \textbf{95.6} \textsuperscript{\textcolor{growthgreen}{\textbf{+1.9}}}
& \textbf{93.8} \textsuperscript{\textcolor{growthgreen}{\textbf{+0.7}}}
& \textbf{96.6} \textsuperscript{\textcolor{growthgreen}{\textbf{+1.8}}}
& HSP-CLIP
& \textbf{93.4} \textsuperscript{\textcolor{growthgreen}{\textbf{+1.6}}}
& \textbf{97.2} \textsuperscript{\textcolor{growthgreen}{\textbf{+6.0}}}
& \textbf{91.8} \textsuperscript{\textcolor{growthgreen}{\textbf{+8.0}}}
& \textbf{95.8} \textsuperscript{\textcolor{growthgreen}{\textbf{+5.9}}}
& \textbf{91.4} \textsuperscript{\textcolor{growthgreen}{\textbf{+4.0}}}
& 88.7
& \textbf{93.1} \textsuperscript{\textcolor{growthgreen}{\textbf{+5.3}}} \\ \bottomrule
\end{tabular}%
}
\end{table*}

\subsection{Integration and Training Objective}
\label{sec:integration_objective}
Finally, HSP is implemented as a lightweight plug-in module and inserted at the $1/4$ and $1/2$ depths of the transformer. The plug-in formulation allows HSP to be integrated into both standard ViT architectures and recent ViT-based deepfake detectors, such as Effort~\cite{effort}, with minimal modification, as demonstrated in Figure~\ref{sec5-fig2}(a). The overall loss function for training the model combines the classification loss $\mathcal{L}_{\text{ce}}$, the reconstruction loss $\mathcal{L}_{\text{rec}}$, the structural regularization terms $\mathcal{L}_{\text{orth}}$ and $\mathcal{L}_{\text{bal}}$,

\begin{equation}
\mathcal{L} =  \mathcal{L}_{\text{ce}} + \lambda_{1} \mathcal{L}_{\text{rec}} + \lambda_{2} \mathcal{L}_{\text{orth}} + \lambda_{3} \mathcal{L}_{\text{bal}},
\end{equation}
where $\lambda_{1}$, $\lambda_{2}$ and $\lambda_{3}$ are hyperparameters that balance the importance of the corresponding terms. 

\section{Experiments}
\subsection{Experimental Setting}
We train all models on the c23 version of FaceForensics++~\cite{faceforensics++}, consistent with the standard detection protocol in DeepfakeBench~\cite{deepfakebench}. We evaluate generalization along two dimensions: \textit{Cross-Dataset Evaluation} and \textit{Cross-Manipulation Evaluation}. We use CLIP ViT-L/14~\cite{clip} as the default backbone and evaluate other ViT-based backbones in Table~\ref{tab:exp-backbone-compat}. Following DeepfakeBench, we sample 8/32 frames for training/inference. AdamW~\cite{adamw} is used with learning rates of $1\times10^{-6}$ for CLIP and $1\times10^{-4}$ for HSP, and batch sizes of 128/256. We apply standard augmentations~\cite{deepfakebench} and report video-level AUC and frame-level AUC/ACC. The default setting uses $K=128$ prototypes, with loss weights $(1,1,1)$ and $(0.5,0.5,0.5)$ for the first and second HSP modules, respectively. More details on the compared baselines are provided in Appendix~\ref{a-b}.

\subsection{Cross-Paradigm Generalization Ability}
\textbf{Cross-Dataset Evaluation.}
We conduct cross-dataset evaluation to assess the generalization capability of HSP-CLIP, as summarized in Table~\ref{tab:exp-cross-dataset}. HSP-CLIP consistently outperforms all state-of-the-art baselines, achieving the highest average AUC at both video and frame levels. At the video level, HSP-CLIP achieves an average AUC of \textbf{94.5\%}, surpassing the strongest competitor ($\chi^2$-DFD) by \textbf{+2.6\%}, with notable improvements on challenging datasets such as DFDC (+3.9\%). At the frame level, HSP-CLIP further achieves an average AUC of \textbf{90.8\%}, outperforming FSgD (85.2\%) by \textbf{+5.6\%} and obtaining consistent gains on Celeb-DF-v2 (+4.0\%), DFDC (+6.2\%), and DFDCP (+4.4\%). 

\noindent \textbf{Cross-Manipulation Evaluation.}
We also conduct cross-manipulation evaluation on DF40, as shown in Table~\ref{tab:exp-cross-manipulation}. HSP-CLIP consistently achieves state-of-the-art performance across six representative face-swapping methods. It reaches an average video-level AUC of \textbf{96.6\%}, outperforming Effort and VbSaT by \textbf{+1.8\%}. At the frame level, HSP-CLIP further achieves an average AUC of \textbf{93.1\%}, exceeding FSgD(87.8\%) by \textbf{+5.3\%}, with substantial improvements on challenging manipulations such as FaceDancer (+8.0\%) and e4s (+6.0\%).

These results demonstrate that HSP-CLIP effectively organizes complex forgery evidence into a diverse and structurally decoupled prototype space, reducing reliance on dataset-specific shortcuts and preserving transferable forensic patterns across diverse manipulation strategies.

% \begin{table*}[t]
% \centering
% % \setlength{\tabcolsep}{2.5pt}
% \caption{Comparison of Accuracy (Acc, \%) for different HSP insertion positions across various datasets.}
% \label{tab:exp-ablation-insert-position}
% \resizebox{\textwidth}{!}{
% \begin{tabular}{lccccccccccc}
% \toprule
% \textbf{Position} & \textbf{CDF-v2} & \textbf{CDF-v1} & \textbf{DFDC} & \textbf{DFDCP} & \textbf{DFD} & \textbf{InSwap} & \textbf{FSGAN} & \textbf{e4s} & \textbf{SimSwap} & \textbf{UniFace} & \textbf{FaceDancer} \\ \midrule
% 1/4                & 80.2 & 81.4 & 74.6 & 77.4 & 87.4 & 78.3 & 90.9 & \textbf{94.7} & 81.4 & 85.5 & 85.3 \\
% 1/2                & \textbf{82.4} & 81.1 & 74.3 & 79.5 & \textbf{90.8} & \textbf{83.5} & 89.7 & 92.1 & 83.1 & 87.7 & \textbf{85.7} \\
% 3/4                & 80.5 & 82.4 & 75.6 & 80.7 & 90.3 & 76.2 & \textbf{91.1} & 94.2 & 82.8 & 87.7 & 84.8 \\
% 1/4+1/2+3/4        & 82.0 & 81.7 & 76.0 & 80.5 & 90.6 & 80.9 & 90.1 & 93.6 & 82.0 & 86.9 & 85.0 \\
% \rowcolor{iceblue}
% \textbf{1/4+1/2}   & 82.3 & \textbf{83.4} & \textbf{77.4} & \textbf{81.0} & 89.3 & 83.2 & 90.3 & 93.2 & \textbf{84.7} & \textbf{89.4} & 85.1 \\ \bottomrule
% \end{tabular}
% }
% \end{table*}
\subsection{Ablation Study}
\textbf{Impact of the HSP Module Placement.}
We compare different HSP placement strategies, as summarized in Figure~\ref{sec5-fig2}(b). We denote the positions at one-quarter, one-half, and three-quarters of the backbone depth as $L_1$, $L_2$, and $L_3$, respectively. Applying HSP at $L_2$ generally outperforms single placement at $L_1$ or $L_3$, while the $L_3$ configuration achieves the weakest performance. Applying HSP at all three positions ($L_1{+}L_2{+}L_3$) also provides no further improvement, whereas the $L_1{+}L_2$ configuration achieves the best overall results. These results suggest that HSP is most effective on shallow-to-intermediate features that preserve fine-grained forensic cues, such as texture inconsistencies, blending boundaries, and geometric distortions. In contrast, early features at $L_1$ contain substantial low-level variations, while deeper features at $L_3$ encode more global semantics and are less suitable for fine-grained forgery-factor modeling. We therefore adopt $L_1{+}L_2$ as the default placement strategy.

\textbf{Incremental Improvement of the Proposed Designs.} We progressively introduce prototype reconstruction $\mathcal{L}_{\mathrm{rec}}$, prototype orthogonality $\mathcal{L}_{\mathrm{orth}}$, and assignment balance $\mathcal{L}_{\mathrm{bal}}$. As shown in Table~\ref{tab:exp-ablation-components}, $\mathcal{L}_{\mathrm{rec}}$ substantially improves the baseline, while $\mathcal{L}_{\mathrm{orth}}$ further improves performance by encouraging factor independence. Adding $\mathcal{L}_{\mathrm{bal}}$ achieves the best results across all benchmarks, confirming that balanced prototype utilization promotes a higher-effective-rank representation and improves cross-domain generalization.

\begin{table}[t]
  \centering
  \renewcommand{\arraystretch}{0.85}
  \caption{Ablation studies (ACC, \%) regarding the proposed HSP-CLIP with loss constraints.}
  \label{tab:exp-ablation-components}
  \small
  \resizebox{0.92\linewidth}{!}{%
  \begin{tabular}{ccc|c|c|c|c}
    \toprule
    $\mathcal{L}_{rec}$ & $\mathcal{L}_{orth}$ & $\mathcal{L}_{bal}$ & CDF-v2 & DFDCP & DFDC & InSwap \\
    \midrule
    \xmark & \xmark & \xmark & 78.6 & 77.9 & 72.8 & 79.8  \\
    \checkmark & \xmark & \xmark & 81.3 & 80.3 & 76.2 & 81.5  \\
    \checkmark & \checkmark & \xmark & 82.1 & 81.0 & 77.1 & 82.1 \\
    % \checkmark & \xmark & \checkmark & 81.5 & 90.2 & 80.3 & 89.2 & 76.8 & 86.6 & 82.2 & 90.3 & 83.5 & 90.4 & 80.9 & 89.3 \\
    \checkmark & \checkmark & \checkmark & 82.8 & 81.6 & 77.5 & 82.9 \\
    \bottomrule
  \end{tabular}
  }
\end{table}

% \begin{table}[t]
% \centering
% \renewcommand{\arraystretch}{0.8}
% \setlength{\tabcolsep}{4.5pt}
% \caption{Video-level AUC (\%) comparison of Effort with and without HSP on selected cross-domain benchmarks.}
% \label{tab:exp-effort-plugin}
% \resizebox{\columnwidth}{!}{%
% \begin{tabular}{lcccccc}
% \toprule
% Method & DFD & DFDC & SimSwap &
% e4s & FSGAN & \textbf{Avg.all} \\
% \midrule
% Effort & 96.5 & 84.3 & 92.6 & 98.3 & 95.7 & 93.5 \\
% \rowcolor{iceblue}
% +HSP & \textbf{98.0} & \textbf{87.8} & \textbf{94.6} &
% \textbf{99.0} & \textbf{98.0} & \textbf{95.3} \\
% \bottomrule
% \end{tabular}%
% }
% \end{table}

\begin{figure}[t]
    \centering
    \includegraphics[width=0.92\linewidth]{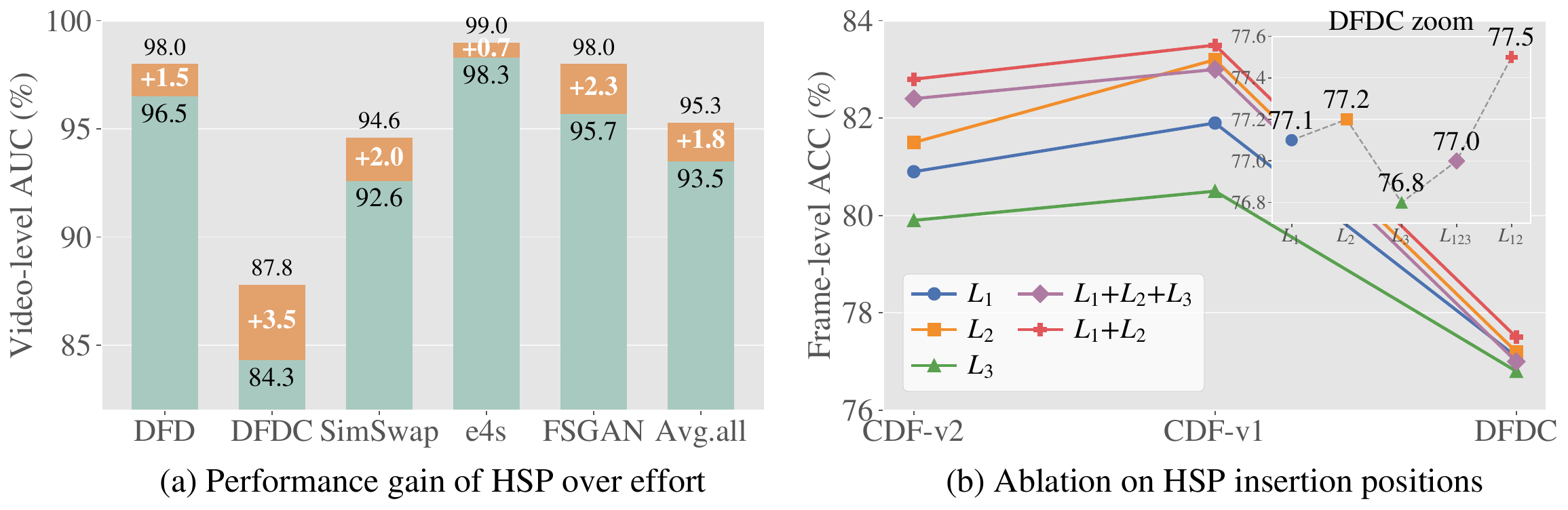}
    \caption{Performance gains of HSP over Effort and ablation study of HSP insertion positions. Figure~(a) and (b) report video-level AUC (\%) and frame-level ACC (\%).}
    \label{sec5-fig2}
\end{figure}

\begin{table}[t]
\caption{Ablation studies (\%) of HSP integrated into different vision foundation models (VFMs).}
\label{tab:exp-backbone-compat}
\centering
\setlength{\tabcolsep}{4pt}
\renewcommand{\arraystretch}{0.7}

\resizebox{0.9\columnwidth}{!}{%
\begin{tabular}{lcccccc}
\toprule
\multirow{2}{*}{Method}
& \multicolumn{2}{c}{CDF-v2}
& \multicolumn{2}{c}{CDF-v1}
& \multicolumn{2}{c}{FSGAN} \\
\cmidrule(r){2-3}
\cmidrule(lr){4-5}
\cmidrule(l){6-7}
& ACC & AUC & ACC & AUC & ACC & AUC \\
\midrule

ViT-Base
& 70.3 & 79.3
& 68.8 & 76.4
& 76.9 & 84.9 \\

+Ours
& 77.4\textsuperscript{\textcolor{growthgreen}{+7.1}}
& 82.7\textsuperscript{\textcolor{growthgreen}{+3.4}}
& 75.6\textsuperscript{\textcolor{growthgreen}{+6.8}}
& 83.5\textsuperscript{\textcolor{growthgreen}{+7.1}}
& 82.7\textsuperscript{\textcolor{growthgreen}{+5.8}}
& 90.4\textsuperscript{\textcolor{growthgreen}{+5.5}} \\
\midrule

SigLIP
& 69.9 & 82.8
& 72.8 & 82.4
& 85.3 & 90.7 \\

+Ours
& 74.9\textsuperscript{\textcolor{growthgreen}{+5.0}}
& 85.9\textsuperscript{\textcolor{growthgreen}{+3.1}}
& 77.2\textsuperscript{\textcolor{growthgreen}{+4.4}}
& 86.1\textsuperscript{\textcolor{growthgreen}{+3.7}}
& 87.9\textsuperscript{\textcolor{growthgreen}{+2.6}}
& 94.4\textsuperscript{\textcolor{growthgreen}{+3.7}} \\
\midrule

CLIP
& 78.6 & 84.5
& 79.2 & 87.4
& 88.7 & 93.9 \\

+Ours
& 82.8\textsuperscript{\textcolor{growthgreen}{+4.2}}
& 90.9\textsuperscript{\textcolor{growthgreen}{+6.4}}
& 83.5\textsuperscript{\textcolor{growthgreen}{+4.3}}
& 91.7\textsuperscript{\textcolor{growthgreen}{+4.3}}
& 90.7\textsuperscript{\textcolor{growthgreen}{+2.0}}
& 95.9\textsuperscript{\textcolor{growthgreen}{+2.0}} \\

\bottomrule
\end{tabular}%
}
\end{table}

% \begin{table}[t]
% \centering
% \renewcommand{\arraystretch}{0.8}
% \setlength{\tabcolsep}{10pt}
% \caption{Ablation study on the insertion positions of HSP. Results are reported in frame-level ACC (\%).}
% \label{tab:exp-ablation-hsp-position}
% \resizebox{\columnwidth}{!}{%
% \begin{tabular}{lcccc}
% \toprule
% \textbf{HSP Position} & CDF-v2 & CDF-v1 & DFDC & FSGAN \\
% \midrule
% $L_1$             & 80.9 & 81.9 & 77.1 & 90.2 \\
% $L_2$             & 81.5 & 83.2 & 77.2 & 90.4 \\
% $L_3$             & 79.9 & 80.5 & 76.8 & 90.0 \\
% $L_1{+}L_2{+}L_3$ & 82.4 & 83.0 & 77.0 & 90.1 \\
% $L_1{+}L_2$       & \textbf{82.8} & \textbf{83.5} & \textbf{77.5} & \textbf{90.7} \\
% \bottomrule
% \end{tabular}%
% }
% \end{table}

\begin{figure}[t]
    \centering
    \includegraphics[width=0.92\linewidth]{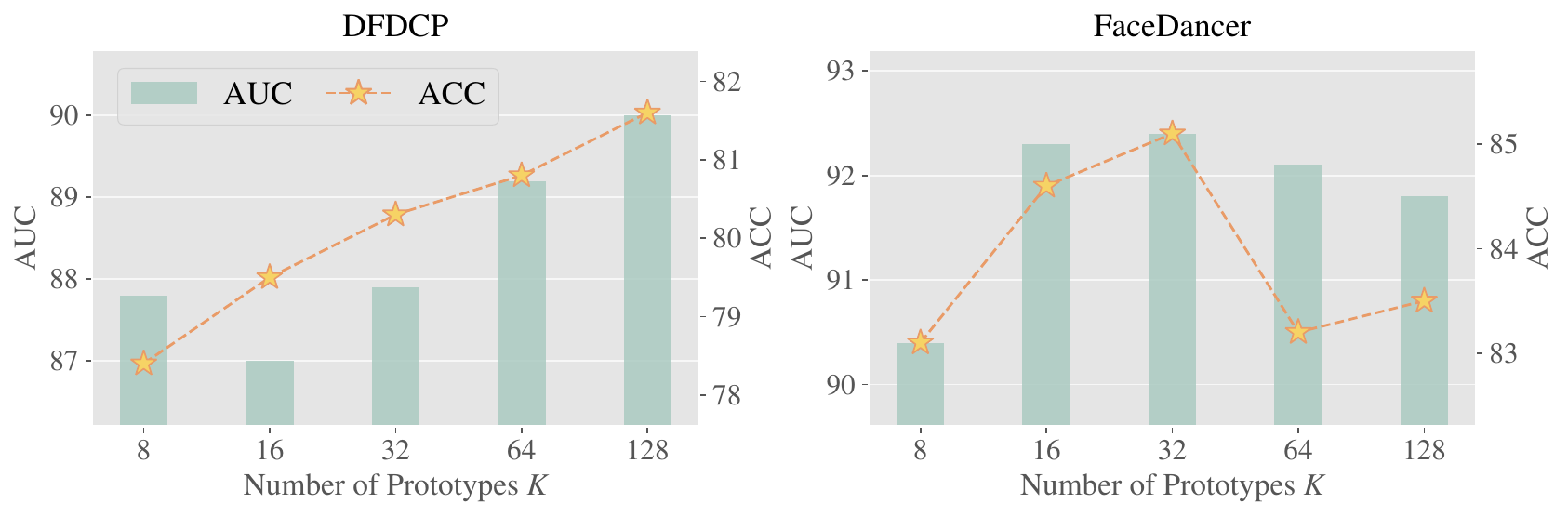}
    \caption{Impact of different numbers of prototypes ($K$).}
    \label{sec5-fig1}
\end{figure}

\textbf{Compatibility with Various Vision Foundation Models.}
We evaluate HSP on three representative backbones: ViT-Base~\cite{vit}, SigLIP~\cite{siglip}, and CLIP~\cite{clip}. As shown in Table~\ref{tab:exp-backbone-compat}, HSP consistently improves both ACC and AUC across all backbones and datasets, demonstrating its compatibility with different vision foundation models. 

\textbf{Compatibility with Existing Detectors.} We further integrate it into Effort~\cite{effort}, a strong recent generalizable deepfake detector, and report video-level AUC results in Figure~\ref{sec5-fig2}(a). HSP improves the average AUC from 93.5\% to 95.3\% and brings consistent gains across the benchmarks. These results indicate that HSP can serve as a general plug-in regularization module for the advanced detectors.

\textbf{Impact of the Number of Prototypes $K$.} We evaluate five prototype settings as shown in Figure~\ref{sec5-fig1}. DFDCP and FaceDancer achieve their best overall performance at $K=128$ and $K=32$, respectively. Complete results are provided in Appendix~\ref{a-d}. Cross-dataset evaluation entails broader shifts in identities and generation pipelines, benefiting from a larger prototype space that more finely models diverse forgery ProDe. In contrast, cross-manipulation evaluation mainly varies the manipulation algorithm while preserving similar facial content, for which moderate capacity is sufficient; larger $K$ may instead introduce redundant directions and weaken generalization. This observation is consistent with our theoretical analysis that increasing $K$ enlarges the maximum effective dimensionality supported by HSP and mitigates collapse into a few source-specific shortcut directions. We adopt $K=128$ as the default setting, prioritizing the more challenging cross-dataset generalization.

\textbf{Additional Analysis.} Appendices~\ref{a-g} and~\ref{a-d} analyze robustness and loss-weight configurations, respectively, while Appendix~\ref{a-e} provides further evidence of balanced prototype utilization, high effective rank, and factor diversity.

% \textbf{Compatibility with Various Vision Foundation Models.} We evaluate HSP on three representative backbones: ViT-Base~\cite{vit}, SigLip~\cite{siglip}, and CLIP~\cite{clip}. As shown in Table~\ref{tab:exp-backbone-compat}, HSP consistently improves both ACC and AUC across all backbones and datasets. With ViT-Base, HSP achieves the largest gain of 7.1 percentage points in ACC on CDF-v2 and improves AUC on CDF-v1 by 7.1 points. For SigLip, it raises CDF-v2 ACC and AUC by 5.0 and 3.1 points, respectively, while reaching 94.4\% AUC on FSGAN. Using CLIP, our default backbone, HSP further improves CDF-v2 ACC and AUC by 4.2 and 6.4 points, and achieves the best overall performance of 95.9\% AUC on FSGAN. These consistent gains demonstrate that HSP is broadly compatible with different vision foundation models.

% \textbf{Compatibility with Existing Detectors. }To further verify the generality of HSP, we integrate it into Effort~\cite{effort}, a strong recent generalizable deepfake detector, and report the video-level AUC results in Table~\ref{tab:exp-effort-plugin}. As shown, HSP improves the average performance from 93.5\% to 95.3\%, with gains observed on most benchmarks. In particular, the improvement is notable on challenging datasets such as CDF-v1, DFDC, FaceDancer, and InSwap. These results indicate that HSP is not restricted to a specific backbone or detector design, but can serve as a general plug-in regularization module for enhancing the transferability of forgery representations.

\section{Conclusion}
\label{sec:conclusion}
\par In this paper, we identify a critical bottleneck in generalizable deepfake detection: \textit{forensic representation collapse}, a failure mechanism where detectors over-index on volatile, dataset-specific shortcuts. To address this, we propose Hierarchical Spatial Prototype (HSP), a novel framework that explicitly models the forgery manifold via a set of learnable orthogonal prototypes. This design increases the effective capacity of the forensic representation from an information-bottleneck perspective, enabling the model to capture a diverse and structurally independent set of forensic cues. It achieves state-of-the-art generalization on prominent benchmarks, advancing the reliability of face forgery detection and helping mitigate the growing risks posed by generative AI.

\bibliography{bib}

\appendix
\clearpage
\section{Theoretical Interpretation of the Diagnostic Metrics}
\label{app:diagnostic_theory}

\begin{itemize}
    \item \textbf{Cluster-level Fake Density ($\mathcal{H}$)} quantifies the concentration of fake samples within each local region of the feature space. It is conceptually related to cluster purity, but is designed to capture whether forgery evidence forms locally coherent and label-enriched forensic regions rather than merely geometrically compact clusters~\cite{davies1979cluster}.
    
    \item \textbf{Cluster--Label Dependence ($\mathcal{I}$)} measures how strongly the discovered feature regions are associated with manipulation labels at the global level. It is closely related in spirit to mutual-information-based dependence measures~\cite{vinh2010information}, and reveals the extent to which the learned cluster structure contributes to the final discrimination between real and fake samples.
    
    \item \textbf{Cross-Dataset Fake Density Correlation ($\mathcal{P}$)} evaluates whether the same forgery-enriched regions remain salient across datasets with different manipulation pipelines and distribution shifts. This metric is intended to quantify structural invariance of forensic patterns across domains~\cite{wang2021oodsurvey}, thereby connecting representation analysis to cross-dataset generalization.
\end{itemize}

Taken together, $\mathcal{H}$, $\mathcal{I}$, and $\mathcal{P}$ capture three complementary aspects of forgery representations. We further provide formal interpretations of each metric and discuss why none is sufficient in isolation.

\paragraph{Notation.}
Let $Y\in\{0,1\}$ denote the forensic label, where $Y=1$ represents a fake sample, and let $C\in\{1,\ldots,N_c\}$ denote the cluster assignment of a spatial token. We denote the global fake prior by $\pi = \Pr(Y=1)$. For class-balanced evaluation, $\pi=1/2$. For the $c$-th feature cluster, its population-level fake density is
\begin{equation}
    \mathcal{H}(c)
    =
    \Pr(Y=1\mid C=c).
    \label{eq:population_fake_density}
\end{equation}
The empirical definition used in the main paper is the corresponding sample estimate of Eq.~\eqref{eq:population_fake_density}.

\subsection{Local Concentration Measured by $\mathcal{H}$}

We first establish the interpretation of $\mathcal{H}$ as a measure of local label enrichment.

\begin{lemma}[Fake density under label-independent clustering]
\label{lem:h_independence}
If the cluster assignment $C$ is statistically independent of the forensic label $Y$, then
$\mathcal{H}(c)=\pi$
for every cluster $c$ with $\Pr(C=c)>0$.
\end{lemma}

\begin{proof}
By statistical independence,
\begin{equation}
\begin{aligned}
\mathcal{H}(c)
&=
\frac{\Pr(Y=1,C=c)}{\Pr(C=c)} \\
&=
\frac{\Pr(Y=1)\Pr(C=c)}{\Pr(C=c)}
=
\Pr(Y=1)
=
\pi.
\end{aligned}
\end{equation}
\end{proof}

Lemma~\ref{lem:h_independence} provides a natural reference level for interpreting fake density. If the learned feature partition contains no label-relevant structure, every nonempty cluster has the same fake proportion as the global class prior $\pi$, which equals $0.5$ under balanced evaluation. Thus, clusters with $\mathcal{H}(c)>\pi$ correspond to local feature regions enriched with forged samples. This interpretation is further supported by
\begin{equation}
\frac{\Pr(C=c\mid Y=1)}
{\Pr(C=c\mid Y=0)}
=
\frac{\mathcal{H}(c)}
{1-\mathcal{H}(c)}
\cdot
\frac{1-\pi}{\pi},
\label{eq:density_likelihood_ratio}
\end{equation}
showing that, for a fixed prior, a larger $\mathcal{H}(c)$ provides stronger local evidence for the fake class. However, $\mathcal{H}$ does not account for cluster probability mass. A small cluster may attain a high fake density while contributing little to the global decision, motivating the complementary use of $\mathcal{I}$.

\subsection{Global Discriminability Measured by $\mathcal{I}$}

The Cluster--Label Dependence is defined as the mutual information between $C$ and $Y$:
\begin{equation}
    \mathcal{I}
    =
    I(C;Y)
    =
    \sum_{c=1}^{N_c}
    \sum_{y\in\{0,1\}}
    p(c,y)
    \log
    \frac{p(c,y)}{p(c)p(y)}.
    \label{eq:mi_definition_appendix}
\end{equation}

The following standard identities clarify its diagnostic meaning:
\begin{equation}
\begin{split}
    I(C;Y)
    &=
    H(Y)-H(Y\mid C) \\
    &=
    \sum_{c=1}^{N_c}
    p(c)
    D_{\mathrm{KL}}
    \left(
        p(Y\mid C=c)
        \,\Vert\,
        p(Y)
    \right).
\end{split}
\label{eq:mi_equivalent_forms}
\end{equation}

\begin{lemma}[Characterization of zero cluster--label dependence]
\label{lem:mi_zero}
The following statements are equivalent:
\begin{enumerate}
    \item $I(C;Y)=0$;
    \item $C$ and $Y$ are statistically independent;
    \item $p(Y\mid C=c)=p(Y)$ for every cluster $c$ with $p(c)>0$.
\end{enumerate}
\end{lemma}

\begin{proof}
Mutual information can be written as
\begin{equation}
    I(C;Y)
    =
    D_{\mathrm{KL}}
    \left(
        p(C,Y)
        \,\Vert\,
        p(C)p(Y)
    \right).
\end{equation}
The non-negativity of the Kullback--Leibler divergence gives
$I(C;Y)\geq 0$, with equality if and only if
\begin{equation}
    p(C,Y)=p(C)p(Y)
\end{equation}
almost everywhere. This is exactly the definition of statistical independence and is equivalent to
$p(Y\mid C=c)=p(Y)$ for all clusters with nonzero probability.
\end{proof}

Equation~\eqref{eq:mi_equivalent_forms} also shows why $\mathcal{I}$ complements $\mathcal{H}$. The fake density of cluster $c$ determines the conditional Bernoulli distribution $p(Y\mid C=c)
    =
    \operatorname{Bernoulli}\bigl(\mathcal{H}(c)\bigr).$ Consequently,
\begin{equation}
    \mathcal{I}
    =
    \sum_{c=1}^{N_c}
    p(c)
    D_{\mathrm{KL}}
    \left(
        \operatorname{Bernoulli}\bigl(\mathcal{H}(c)\bigr)
        \,\Vert\,
        \operatorname{Bernoulli}(\pi)
    \right).
    \label{eq:mi_density_relation}
\end{equation}
Thus, $\mathcal{I}$ aggregates the deviation of cluster-specific fake densities from the global class prior, weighted by the occurrence probability $p(c)$. A small but highly fake-enriched cluster has limited influence, whereas sufficiently populated real- and fake-enriched regions contribute to stronger global dependence. Moreover, since $I(C;Y)=H(Y)-H(Y\mid C)$, larger $\mathcal{I}$ indicates that feature-region assignments reduce uncertainty about forensic labels. Therefore, $\mathcal{I}$ measures the global discriminative relevance of the learned partition rather than the purity of individual regions.

\subsection{Cross-Domain Stability Measured by $\mathcal{P}$}

Let $\mathcal{D}_s$ and $\mathcal{D}_t$ denote the source and target domains, respectively. The cluster centroids are estimated only from $\mathcal{D}_s$ and are then fixed when assigning target-domain tokens. This produces aligned fake-density profiles
\begin{equation}
    \mathbf{V}_s
    =
    \left[
        \mathcal{H}_s(1),
        \ldots,
        \mathcal{H}_s(N_c)
    \right]^\top,
\end{equation}
\begin{equation}
    \mathbf{V}_t
    =
    \left[
        \mathcal{H}_t(1),
        \ldots,
        \mathcal{H}_t(N_c)
    \right]^\top.
\end{equation}
Because both domains use the same source-derived centroids, each entry corresponds to the same source-defined feature region. Let $\mathcal{M}$ denote the top-$M$ clusters with the largest source-domain fake densities. We define $\mathbf{v}_s=\mathbf{V}_s[\mathcal{M}]$ and $\mathbf{v}_t=\mathbf{V}_t[\mathcal{M}]$. When both vectors have nonzero variance, their Pearson correlation is
\begin{equation}
    \mathcal{P}
    =
    \frac{
        \left\langle
            \mathbf{v}_s-\bar v_s\mathbf{1},
            \mathbf{v}_t-\bar v_t\mathbf{1}
        \right\rangle
    }{
        \left\|
            \mathbf{v}_s-\bar v_s\mathbf{1}
        \right\|_2
        \left\|
            \mathbf{v}_t-\bar v_t\mathbf{1}
        \right\|_2
    }.
    \label{eq:p_appendix}
\end{equation}

\begin{proposition}[Affine structural invariance]
\label{prop:affine_invariance}
Suppose the target-domain fake-density profile satisfies
$\mathbf{v}_t=a\mathbf{v}_s+b\mathbf{1}$,
where $a\neq0$ and $\mathbf{v}_s$ has nonzero variance. Then $\mathcal{P}=\operatorname{sign}(a)$, i.e., $\mathcal{P}=1$ for $a>0$ and $\mathcal{P}=-1$ for $a<0$.
\end{proposition}

\begin{proof}
The centered target vector satisfies
\begin{align}
\mathbf{v}_t-\bar v_t\mathbf{1}
&=
a(\mathbf{v}_s-\bar v_s\mathbf{1}).
\end{align}
Substituting into Eq.~\eqref{eq:p_appendix} yields $\mathcal{P}=\frac{a}{|a|}=\operatorname{sign}(a)$.
\end{proof}

Proposition~\ref{prop:affine_invariance} shows that $\mathcal{P}$ measures the preservation of the \emph{relative organization} of forgery-enriched regions rather than identical density magnitudes across domains. A target domain may exhibit lower fake densities due to increased difficulty, yet $\mathcal{P}$ remains high if source-dominant fake-enriched clusters preserve their relative ordering. Conversely, a low or negative $\mathcal{P}$ indicates that source-domain forgery regions lose their association with fake samples after domain shift, suggesting reliance on unstable domain-specific shortcuts. The top-$M$ restriction further focuses the analysis on the most forgery-relevant regions. Computing correlation over all clusters may be dominated by semantic or label-neutral regions with $\mathcal{H}_s(c)\approx\mathcal{H}_t(c)\approx\pi$, which provide limited forensic information. Selecting $\mathcal{M}$ from the source domain therefore highlights source-derived forgery evidence while avoiding target-domain information leakage.

% \begin{remark}[Degenerate density profiles]
% If either $\mathbf{v}_s$ or $\mathbf{v}_t$ has zero variance, Pearson correlation is undefined. In practice, this corresponds to a degenerate profile in which all selected clusters have identical fake densities. Such a representation contains no meaningful relative density structure over the selected regions. These cases should therefore be treated as structurally non-informative rather than being interpreted as evidence of high cross-domain stability.
% \end{remark}

\begin{figure*}[t]
    \centering
    \includegraphics[width=0.85\linewidth]{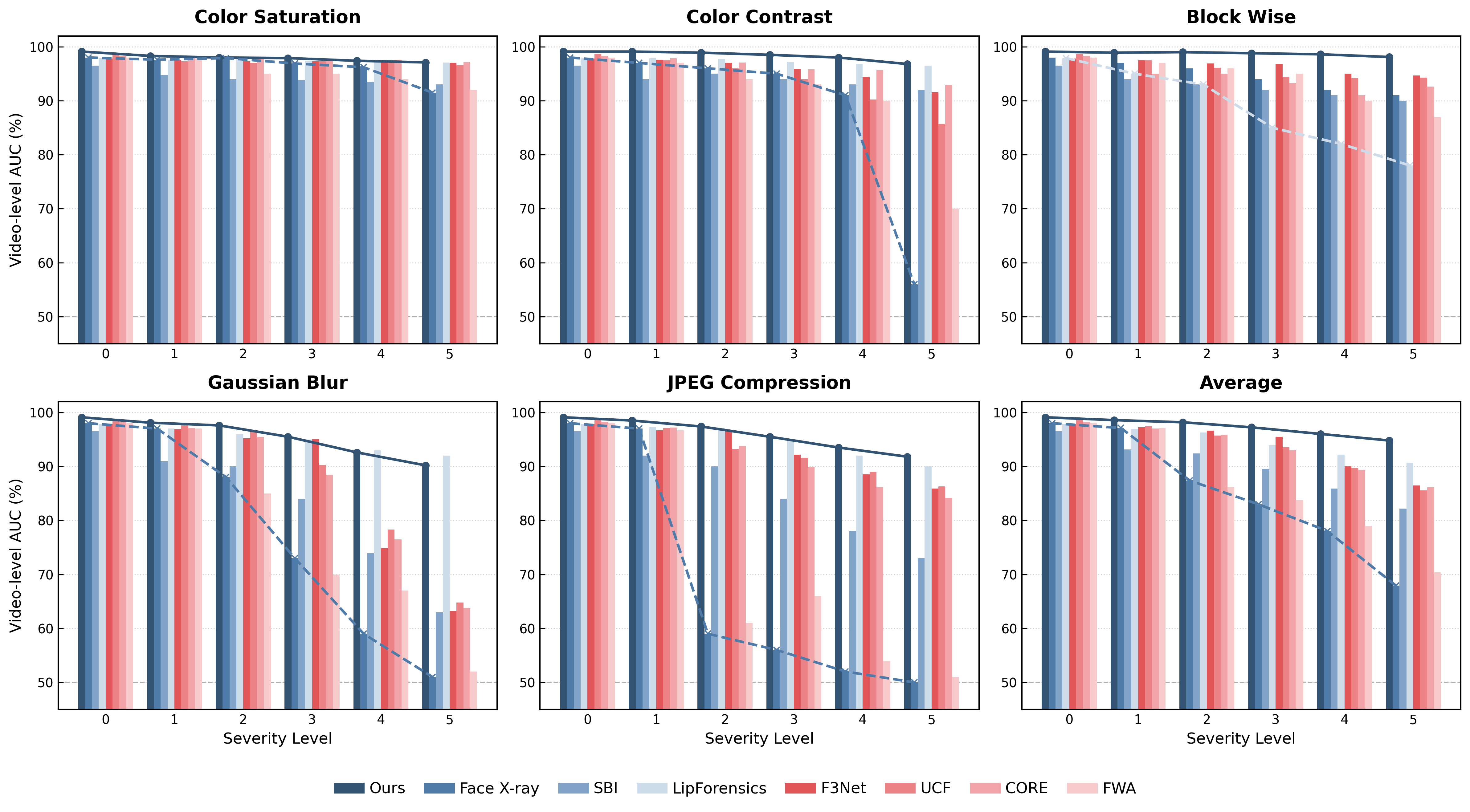}
    \caption{Robustness evaluation under various unseen perturbations. The grouped bar charts illustrate the video-level AUC (\%) performance of our method versus competing approaches across five different perturbation types: Color Saturation, JPEG Compression, Color Contrast, Block Wise, and Gaussian Blur. Performance is plotted against increasing severity levels from 0 to 5. Overall, our method demonstrates the strongest comprehensive robustness, achieving highly competitive performance across different unseen perturbation types.}
    \label{fig:exp-robustness}
\end{figure*}

\subsection{The Complementarity of Three Diagnostics}
We then show through counterexamples that no individual diagnostic is sufficient to characterize a transferable forensic representation.

\begin{proposition}[Non-sufficiency of individual diagnostics]
\label{prop:non_sufficiency}
Each of $\mathcal{H}$, $\mathcal{I}$, and $\mathcal{P}$ can attain a high value while the representation fails to satisfy at least one other desirable property of a transferable forensic representation.
\end{proposition}

\begin{proof}
We consider three constructions.

\paragraph{High $\mathcal{H}$ but low $\mathcal{I}$.}
Suppose one cluster contains only fake tokens but has probability mass $\epsilon$, where $\epsilon\rightarrow 0$. Then its fake density satisfies $\mathcal{H}(c)=1.$ However, its contribution to Eq.~\eqref{eq:mi_density_relation} is weighted by $p(c)=\epsilon$ and therefore vanishes as $\epsilon\rightarrow 0$. Hence, a locally pure fake cluster need not induce globally discriminative structure.

\paragraph{High $\mathcal{H}$ and $\mathcal{I}$ but low $\mathcal{P}$.}
Suppose the source domain contains an artifact that perfectly separates fake and real samples, producing highly label-enriched and sufficiently populated clusters. Both $\mathcal{H}_s$ and $I_s(C;Y)$ can then be high. If the artifact is absent or reverses its label association in the target domain, the source-enriched clusters will not remain target-enriched, resulting in a low or negative $\mathcal{P}$. This construction corresponds to source-domain shortcut learning.

\paragraph{High $\mathcal{P}$ but low $\mathcal{H}$ and $\mathcal{I}$.}
Consider nonconstant vectors $\boldsymbol{\delta}_s$ and $\boldsymbol{\delta}_t$ satisfying
$\boldsymbol{\delta}_t=a\boldsymbol{\delta}_s$ for some $a>0$, and define
\begin{equation}
    \mathbf{v}_s
    =
    \pi\mathbf{1}
    +
    \epsilon\boldsymbol{\delta}_s,
    \qquad
    \mathbf{v}_t
    =
    \pi\mathbf{1}
    +
    \epsilon\boldsymbol{\delta}_t,
\end{equation}
where $\epsilon>0$ is arbitrarily small. Proposition~\ref{prop:affine_invariance} gives $\mathcal{P}=1$. However, as $\epsilon\rightarrow 0$, all cluster-specific fake densities converge to the class prior $\pi$. By Eq.~\eqref{eq:mi_density_relation}, $\mathcal{I}\rightarrow 0$, and the clusters become nearly label-neutral. Thus, high correlation alone does not imply useful forensic discrimination.
\end{proof}

These cases motivate the following joint interpretation.

\begin{proposition}[Joint diagnostic interpretation]
\label{prop:joint_interpretation}
Jointly high values of $\mathcal{H}$, $\mathcal{I}$, and $\mathcal{P}$ indicate three complementary representation properties:
\begin{enumerate}
    \item high $\mathcal{H}$ indicates locally fake-enriched feature regions;
    \item high $\mathcal{I}$ indicates strong dependence between the feature partition and forensic labels;
    \item high $\mathcal{P}$ indicates stable cross-domain organization of forgery-enriched regions.
\end{enumerate}
\end{proposition}

High $\mathcal{H}$, $\mathcal{I}$, and $\mathcal{P}$ do not guarantee low target-domain classification error. Instead, they capture complementary and theoretically interpretable properties expected in transferable forensic representations: \textit{local label enrichment} ($\mathcal{H}$), \textit{global discriminability} ($\mathcal{I}$), and \textit{cross-domain structural stability} ($\mathcal{P}$). Together, these diagnostics distinguish stable forgery-related structures from fragile dataset-specific shortcuts.

\section{The Impact of Different $M$ on Diagnostic Metrics}
\label{a-c}
We further evaluate the robustness of the proposed diagnostics by varying the number of selected fake-dominated clusters as $M\in\{10,20,30\}$. As shown in Table~\ref{tab:analysis-metrics-m}, HSP-CLIP consistently achieves higher cluster-level fake density $\bar{\mathcal{H}}$ and cross-dataset fake density correlation $\mathcal{P}$ than CLIP-FT across all target datasets and values of $M$, indicating more discriminative and cross-domain-stable forensic structures. The improvement in $\mathcal{P}$ is most pronounced at small $M$ and gradually narrows as more clusters are included. This behavior is consistent with the different representation structures learned by the two methods. CLIP-FT concentrates forgery evidence along a few dominant directions, making its most fake-enriched clusters more likely to encode source-specific shortcuts. Because a small $M$ primarily evaluates these core clusters, CLIP-FT exhibits relatively weak cross-domain stability. In contrast, HSP reduces reliance on a few source-specific directions through multi-factor forgery-space modeling, substantially improving the consistency of dominant forensic regions across domains. As $M$ increases, additional clusters with weaker discriminative relevance or generic visual content are included, which may dilute the sensitivity of $\mathcal{P}$ to shortcut-induced instability and narrow the gap between HSP-CLIP and CLIP-FT.

\begin{table}[t]
    \centering
    \small
    \caption{Quantitative comparison of the feature representations learned by
    CLIP-FT and HSP-CLIP under different values of $M$, measured by
    cluster-level fake density $\bar{\mathcal{H}}$ and cross-dataset fake
    density correlation $\mathcal{P}$.}
    \label{tab:analysis-metrics-m}
    \renewcommand{\arraystretch}{0.9}
    \setlength{\tabcolsep}{2pt}

    \resizebox{\columnwidth}{!}{%
    \begin{tabular}{lccc|ccc}
        \toprule
        \multirow{2}{*}{\textbf{Method}}
        & \multicolumn{3}{c|}{$\bar{\mathcal{H}}$}
        & \multicolumn{3}{c}{$\mathcal{P}$} \\
        \cmidrule(lr){2-4}
        \cmidrule(lr){5-7}
        & $M=10$ & $M=20$ & $M=30$
        & $M=10$ & $M=20$ & $M=30$ \\
        \midrule

        \multicolumn{7}{c}{\textbf{Celeb-DF-v2}} \\
        \midrule
        CLIP-FT
        & 0.996 & 0.973 & 0.907
        & 0.483 & 0.853 & 0.914 \\
        HSP-CLIP-FT
        & \textbf{0.998} & \textbf{0.985} & \textbf{0.936}
        & \textbf{0.832} & \textbf{0.960} & \textbf{0.934} \\
        \midrule

        \multicolumn{7}{c}{\textbf{DFDCP}} \\
        \midrule
        CLIP-FT
        & 0.965 & 0.950 & 0.861
        & 0.544 & 0.914 & 0.942 \\
        HSP-CLIP-FT
        & \textbf{0.993} & \textbf{0.968} & \textbf{0.883}
        & \textbf{0.917} & \textbf{0.976} & \textbf{0.979} \\
        \midrule

        \multicolumn{7}{c}{\textbf{DFDC}} \\
        \midrule
        CLIP-FT
        & 0.970 & 0.940 & 0.857
        & 0.523 & 0.748 & 0.913 \\
        HSP-CLIP-FT
        & \textbf{0.993} & \textbf{0.957} & \textbf{0.876}
        & \textbf{0.817} & \textbf{0.977} & \textbf{0.959} \\
        \bottomrule
    \end{tabular}%
    }
\end{table}

\begin{table*}[t]
\centering
\renewcommand{\arraystretch}{0.9}
\setlength{\tabcolsep}{5pt}
\caption{Performance comparison across different numbers of prototypes. Results are reported in frame-level ACC and AUC (\%). The best result for each dataset and metric is shown in \textbf{bold}. Blue cells highlight the selected high-performance region.}
\label{tab:exp-ablation-prototypes}
\resizebox{0.9\textwidth}{!}{%
\begin{tabular}{cc|ccccccccccc}
\toprule
\textbf{$K$} & \textbf{Metric} & \textbf{CD2} & \textbf{CD1} & \textbf{DFDCP} & \textbf{DFD} & \textbf{DFDC} & \textbf{SimSwap} & \textbf{UniFace} & \textbf{e4s} & \textbf{FaceDancer} & \textbf{FSGAN} & \textbf{InSwap} \\
\midrule
\multirow{2}{*}{8}
& ACC & 79.1 & 78.9 & 78.4 & \textbf{91.2} & 76.8 & 83.0 & 87.2 & 90.3 & 83.1 & 89.1 & 83.2 \\
& AUC & 89.5 & 90.6 & 87.8 & 94.1 & \textbf{87.2} & 89.2 & 93.9 & 96.8 & 90.4 & 95.5 & 90.3 \\
\midrule
\multirow{2}{*}{16}
& ACC & 81.6 & 81.3 & 79.5 & 90.9 & 76.3 & \textbf{85.0} & 87.7 & 90.1 & 84.6 & 90.2 & \textbf{84.2} \\
& AUC & 89.8 & 91.3 & 87.0 & 94.1 & 86.2 & \textbf{91.9} & \textbf{95.1} & 97.0 & 92.3 & 95.9 & \textbf{91.6} \\
\midrule
\multirow{2}{*}{32}
& ACC & 82.0 & 82.9 & 80.3 & 90.9 & 76.7 & 83.6 & \textbf{89.1} & \textbf{93.1} & \textbf{85.1} & 90.4 & 82.4 \\
& AUC & 90.3 & 89.7 & 87.9 & 94.2 & 85.6 & 90.2 & 94.7 & \textbf{97.7} & \textbf{92.4} & \textbf{96.3} & 90.4 \\
\midrule
\multirow{2}{*}{64}
& ACC & 82.3 & 82.7 & 80.8 & 89.5 & 76.9 & 81.1 & 87.6 & 91.4 & 83.2 & 90.1 & 80.9 \\
& AUC & 89.8 & 91.2 & 89.2 & 93.9 & 85.4 & 90.0 & 94.3 & 97.0 & 92.1 & 95.9 & 90.7 \\
\midrule
\multirow{2}{*}{128}
& ACC
& \cellcolor{iceblue}\textbf{82.8}
& \cellcolor{iceblue}\textbf{83.5}
& \cellcolor{iceblue}\textbf{81.6}
& \cellcolor{iceblue}90.1
& \cellcolor{iceblue}\textbf{77.5}
& \cellcolor{iceblue}81.6
& \cellcolor{iceblue}86.0
& \cellcolor{iceblue}91.1
& \cellcolor{iceblue}83.5
& \cellcolor{iceblue}\textbf{90.7}
& \cellcolor{iceblue}82.9 \\

& AUC
& \cellcolor{iceblue}\textbf{90.9}
& \cellcolor{iceblue}\textbf{91.9}
& \cellcolor{iceblue}\textbf{90.0}
& \cellcolor{iceblue}\textbf{94.3}
& \cellcolor{iceblue}86.7
& \cellcolor{iceblue}88.7
& \cellcolor{iceblue}93.4
& \cellcolor{iceblue}97.2
& \cellcolor{iceblue}91.8
& \cellcolor{iceblue}95.8
& \cellcolor{iceblue}91.4 \\
\bottomrule
\end{tabular}%
}
\end{table*}

\section{Evaluation Benchmarking}
\label{a-b}
To provide a comprehensive benchmark, we evaluate our approach against a diverse array of competitive baselines, including F3Net~\cite{f3net}, Face X-Ray~\cite{face-x-ray}, RECCE~\cite{recce}, and SBIs~\cite{sbi}, LaV~\cite{lav}, CFM~\cite{cfm}, IID~\cite{iid}, UCF~\cite{ucf}, FakeRadar~\cite{fakeradar}, LSDA~\cite{lsda}, FCG~\cite{fcg}, ProDet~\cite{prodet}, VbSaT~\cite{vbsat}, DeepShield~\cite{deepshield}, $\chi^2$-DFD~\cite{x2dfd}, FreqDebias~\cite{freqdebias}, ExposDe~\cite{exposde}, FSgD~\cite{fsgd}, RAID~\cite{raid}, UDD~\cite{udd}, DFFake~\cite{dffake}, ICT~\cite{ict}, CDFA~\cite{cdfa}, and Effort ~\cite{effort}. For a fair and rigorous comparison, all models are trained on the FaceForensics++ (c23) dataset and evaluated across multiple unseen datasets and manipulation techniques to assess their cross-domain generalization.

\begin{table*}[t]
\centering
\renewcommand{\arraystretch}{0.92}
\setlength{\tabcolsep}{3.8pt}
\caption{Ablation study on layer-wise loss weighting. \emph{Uniform Weighting} assigns identical loss weights to both HSP layers, whereas \emph{Hierarchical Weighting} reduces the loss weights of the second HSP layer by half. Results are reported in frame-level ACC and AUC (\%), with the best result for each dataset and metric shown in \textbf{bold}.}
\label{tab:exp-ablation-loss-weight}
\resizebox{0.9\textwidth}{!}{%
\begin{tabular}{cc|ccccccccccc}
\toprule
\textbf{Weighting Strategy} & \textbf{Metric} & \textbf{CD2} & \textbf{CD1} & \textbf{DFDCP} & \textbf{DFD} & \textbf{DFDC} & \textbf{SimSwap} & \textbf{UniFace} & \textbf{e4s} & \textbf{FaceDancer} & \textbf{FSGAN} & \textbf{InSwap} \\
\midrule
\multirow{2}{*}{Uniform} & ACC & 82.2 & \textbf{83.8} & \textbf{81.8} & \textbf{92.2} & 76.8 & 80.7 & 85.1 & 89.8 & \textbf{83.5} & 90.6 & 82.1 \\
& AUC & 90.6 & 91.8 & 88.9 & 93.7 & 85.9 & 87.9 & 92.7 & 97.0 & 91.7 & 95.5 & 90.2 \\
\midrule
\multirow{2}{*}{Hierarchical} & ACC & \textbf{82.8} & 83.5 & 81.6 & 90.1 & \textbf{77.5} & \textbf{81.6} & \textbf{86.0} & \textbf{91.1} & \textbf{83.5} & \textbf{90.7} & \textbf{82.9} \\
& AUC & \textbf{90.9} & \textbf{91.9} & \textbf{90.0} & \textbf{94.3} & \textbf{86.7} & \textbf{88.7} & \textbf{93.4} & \textbf{97.2} & \textbf{91.8} & \textbf{95.8} & \textbf{91.4} \\
\bottomrule
\end{tabular}%
}
\end{table*}

\section{Robustness to Image Degradations}
\label{a-g}
% \subsection{Robustness to Image Degradations}
\par To evaluate the robustness of HSP against common image degradations, we adopt the evaluation protocols from previous studies~\cite{lips,effort,lsda,df1.0}. Specifically, we examine five distinct types of perturbations: Color Saturation Change, JPEG Compression, Color Contrast Change, Block-wise Distortion, and Gaussian Blur. For each perturbation, severity levels ranging from 0 (clean) to 5 are applied to assess model stability under varying degrees of degradation. All models are trained on FaceForensics++ (c23) and evaluated based on video-level AUC scores. We compare our method against several state-of-the-art approaches, including Face X-ray~\cite{face-x-ray}, SBI~\cite{sbi}, LipForensics~\cite{lips}, and FWA~\cite{fwa}. Additionally, for a strictly fair comparison, we reproduce F3Net~\cite{f3net}, UCF~\cite{ucf}, and CORE~\cite{core}, ensuring they utilize \textbf{the same data augmentation scheme} as our proposed method. The quantitative results under these unseen perturbations are illustrated in Figure~\ref{fig:exp-robustness}. As observed, while most competing methods suffer significant performance drops under severe perturbations, particularly in JPEG compression and Gaussian blur where high-frequency details are lost, our method maintains consistently high performance across all severity levels. Even compared to the reproduced baselines~\cite{f3net, ucf, core} trained with identical data augmentation, HSP-CLIP demonstrates superior stability, validating that its robustness stems from the learned generalized features rather than merely data augmentation.

\section{Additional Ablation Studies}
\label{a-d}
\subsection{Impact of the Number of Prototypes $K$}
We vary the number of prototypes within $\{8,16,32,64,128\}$, and the results are shown in Table~\ref{tab:exp-ablation-prototypes}. On cross-dataset benchmarks (\textit{e.g.}, CD2, CD1, DFDCP, and DFDC), performance generally improves with a larger prototype space, achieving the best results at $K=128$. In contrast, under cross-manipulation evaluation (\textit{e.g.}, DF40), performance peaks at $K=32$ and slightly decreases with further expansion. This trend is consistent with the experimental conclusions reported in the main text.

\begin{figure*}[t]
    \centering
    \includegraphics[width=0.96\linewidth]{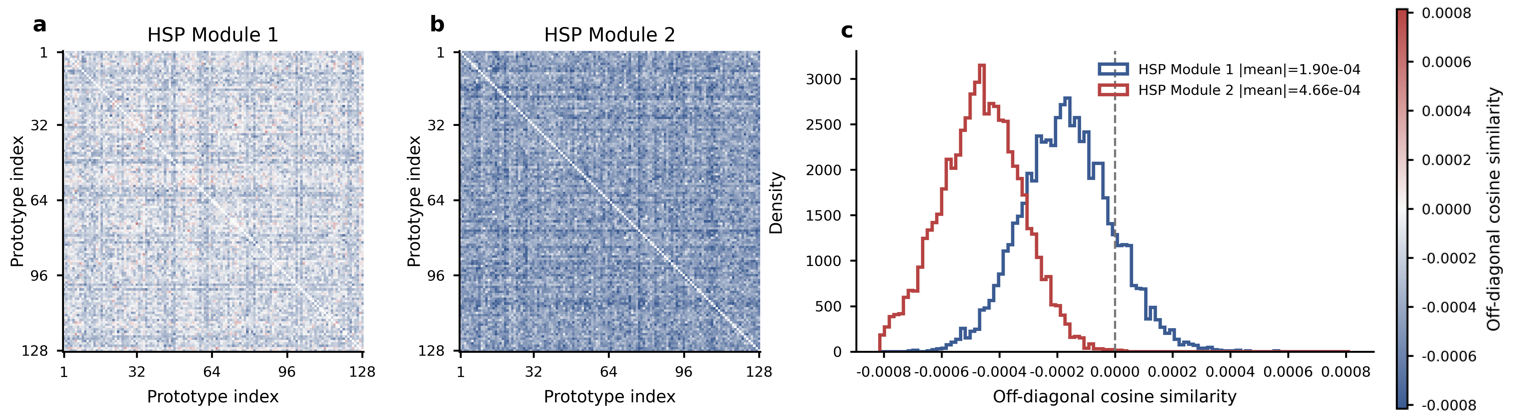}
    \caption{Cosine similarity analysis of the learned spatial prototypes in the two HSP modules ($K=128$). (a,b) Prototype cosine similarity matrices, where the off-diagonal entries are concentrated near zero. (c) Distributions of off-diagonal cosine similarities, with mean absolute values of $1.90\times10^{-4}$ and $4.66\times10^{-4}$ for Modules 1 and 2, respectively. These near-zero similarities demonstrate that HSP learns geometrically decoupled prototypes with minimal redundancy.}
    \label{fig:exp-prototype-diversity}
\end{figure*}

\begin{figure*}[t]
    \centering
    \includegraphics[width=0.98\linewidth]{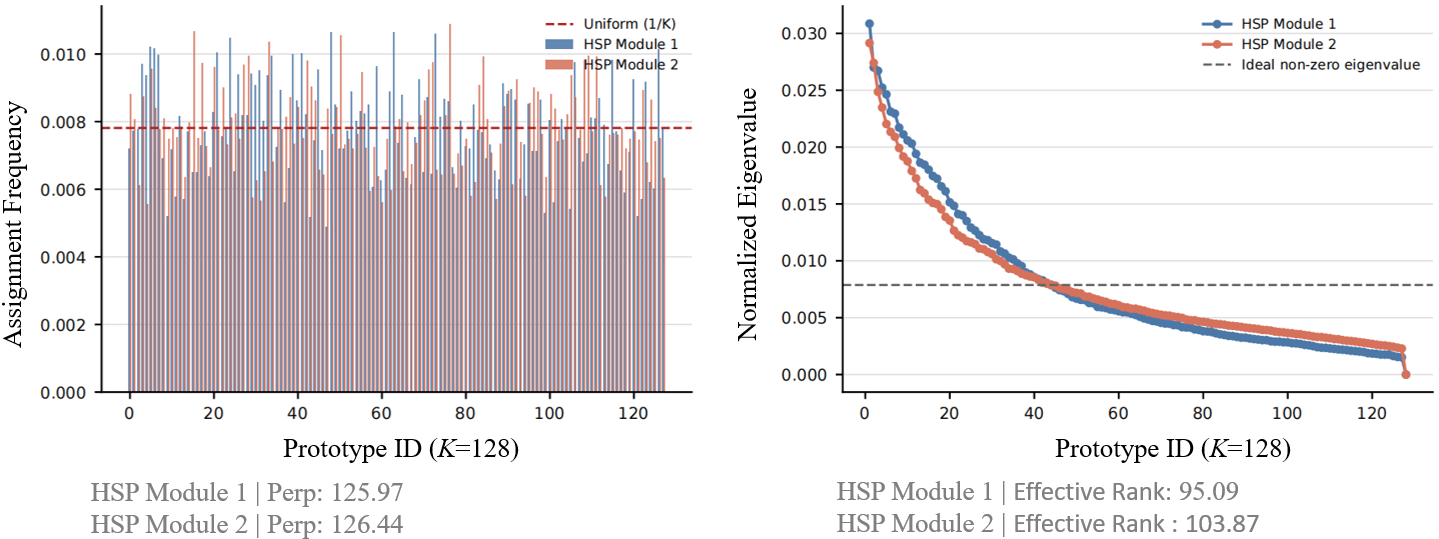}
    \caption{Prototype assignment balance across the two HSP modules ($K=128$) on CDF-v2. The left panel shows the prototype-usage frequencies, with the red dashed line denoting the uniform baseline ($1/K$); both modules achieve \textit{Perplexity} scores close to the theoretical maximum of 128. The right panel shows the normalized eigenvalue spectra of the centered assignment covariance, where the effective ranks reach 95.09 and 103.87 and approach the maximum admissible value of 127, indicating broad utilization of the admissible assignment subspace.}
    \label{fig:exp-assignment-balance}
\end{figure*}

\subsection{Impact of Hierarchical Loss Weighting}
To examine whether different HSP layers require different regularization strengths, we compare two loss-weighting strategies. \emph{Uniform Layer-wise Weighting} assigns identical weights to all HSP layers, while \emph{Hierarchical Layer-wise Weighting} reduces the loss weights of the second layer by half. As shown in Table~\ref{tab:exp-ablation-loss-weight}, hierarchical weighting achieves more consistent performance across most benchmarks, while uniform weighting is only beneficial on a few datasets. This behavior stems from the hierarchical nature of forensic representations. Earlier layers preserve fine-grained forgery cues, such as texture inconsistencies and blending boundaries, and are therefore more suitable for factor modeling and decoupling. In contrast, deeper layers encode more global semantic information, where excessive reconstruction and decoupling constraints may restrict their discriminative capacity. Reducing the regularization strength of the second HSP layer thus preserves high-level semantics while maintaining structured modeling of fine-grained cues. This observation is consistent with the layer-placement results in Figure~\ref{sec5-fig2}(b), where deeper features are less suitable for fine-grained factor decoupling.

\section{Analysis of Internal Properties Explaining HSP Effectiveness}
\label{a-e}
\subsection{Prototype Diversity Analysis.}
We assess the diversity of the learned spatial prototypes under the dual-layer HSP configuration with $K=128$. For each module, we compute the cosine similarity matrix $\mathbf S=\hat{\mathbf P}\hat{\mathbf P}^{\top}$, where $\hat{\mathbf P}$ is obtained by row-wise $\ell_2$ normalization. As shown in Figure~\ref{fig:exp-prototype-diversity}(a,b), both similarity matrices exhibit clear identity-like structures, with diagonal entries of 1.0 and off-diagonal values concentrated near zero. Figure~\ref{fig:exp-prototype-diversity}(c) further quantifies this behavior: the mean absolute off-diagonal cosine similarities are only $1.90\times10^{-4}$ and $4.66\times10^{-4}$ for Modules 1 and 2, respectively. These results confirm that the orthogonality constraint learns nearly independent prototype directions with minimal redundancy, effectively preventing prototype collapse.

\subsection{Prototype Assignment Analysis. }
To validate the \textit{assignment balance} constraint, we randomly sample 2,000 real and 2,000 fake images from CDF-v2 and extract token-to-prototype soft assignments from the two HSP modules with $K=128$. We evaluate assignment balance from two complementary perspectives. First, we average the assignments over all tokens to obtain the marginal prototype-usage distribution $\bar{\mathbf a}$ and measure its uniformity using Perplexity:
\[
\mathrm{Perp}
=
\exp\left(
-\sum_{k=1}^{K}
\bar{a}_k\log(\bar{a}_k+\epsilon)
\right),
\]
whose maximum is $K$ under perfectly uniform usage. Second, we center the assignment matrix as $\mathbf A_c=\mathbf A-\mathbf 1\bar{\mathbf a}^{\top}$ and compute the normalized eigenvalues of $\mathbf A_c^{\top}\mathbf A_c$ as $\lambda_k=\sigma_k/\sum_j\sigma_j$. We then quantify the diversity of assignment variations using the effective rank:
\[
r_{\mathrm{eff}}
=
\exp\left(
-\sum_{k=1}^{K}
\lambda_k\log(\lambda_k+\epsilon)
\right).
\]
Because each assignment vector sums to one, the centered assignments lie in a $(K-1)$-dimensional subspace, yielding a maximum admissible effective rank of $K-1=127$.

As shown in Figure~\ref{fig:exp-assignment-balance}, the prototype-usage frequencies of both modules closely follow the uniform baseline $1/K$. Module 1 and Module 2 achieve Perplexity scores of $125.97$ and $126.44$, respectively, approaching the theoretical maximum of $128$ and indicating broad prototype utilization. Their effective ranks reach $95.09$ and $103.87$, corresponding to $74.8\%$ and $81.8\%$ of the admissible maximum, respectively. Both values approach the theoretical maximum of $127$, particularly for Module 2, indicating that assignment variations span a broad set of valid directions, thereby enabling the model to learn a richer set of latent forgery-feature directions.
% To validate the effectiveness of our Assignment Balance constraint, we investigate the usage distribution of spatial prototypes. Specifically, we randomly sample 2,000 real and 2,000 fake images from CDF-v2 and extract token-to-prototype soft assignments from the two HSP modules ($K=32$). We quantify assignment uniformity using Perplexity (Perp), defined as the exponentiated entropy of the assignment distribution, and KL divergence with respect to the ideal uniform distribution ($1/K$). Ideally, a perfectly balanced assignment achieves zero KL divergence and maximum Perplexity of $K$. As illustrated in Figure~\ref{fig:exp-assignment-balance}, both HSP modules closely follow the uniform baseline. Quantitatively, Module 1 and Module 2 achieve KL divergences of 0.0502 and 0.0377, with Perplexity scores of 30.57 and 30.83, respectively, approaching the theoretical maximum of 32. These results confirm that assignment balance effectively prevents prototype collapse and encourages diverse utilization of the learned forgery factors.

\begin{figure*}[t]
    \centering
    \includegraphics[width=0.85\linewidth]{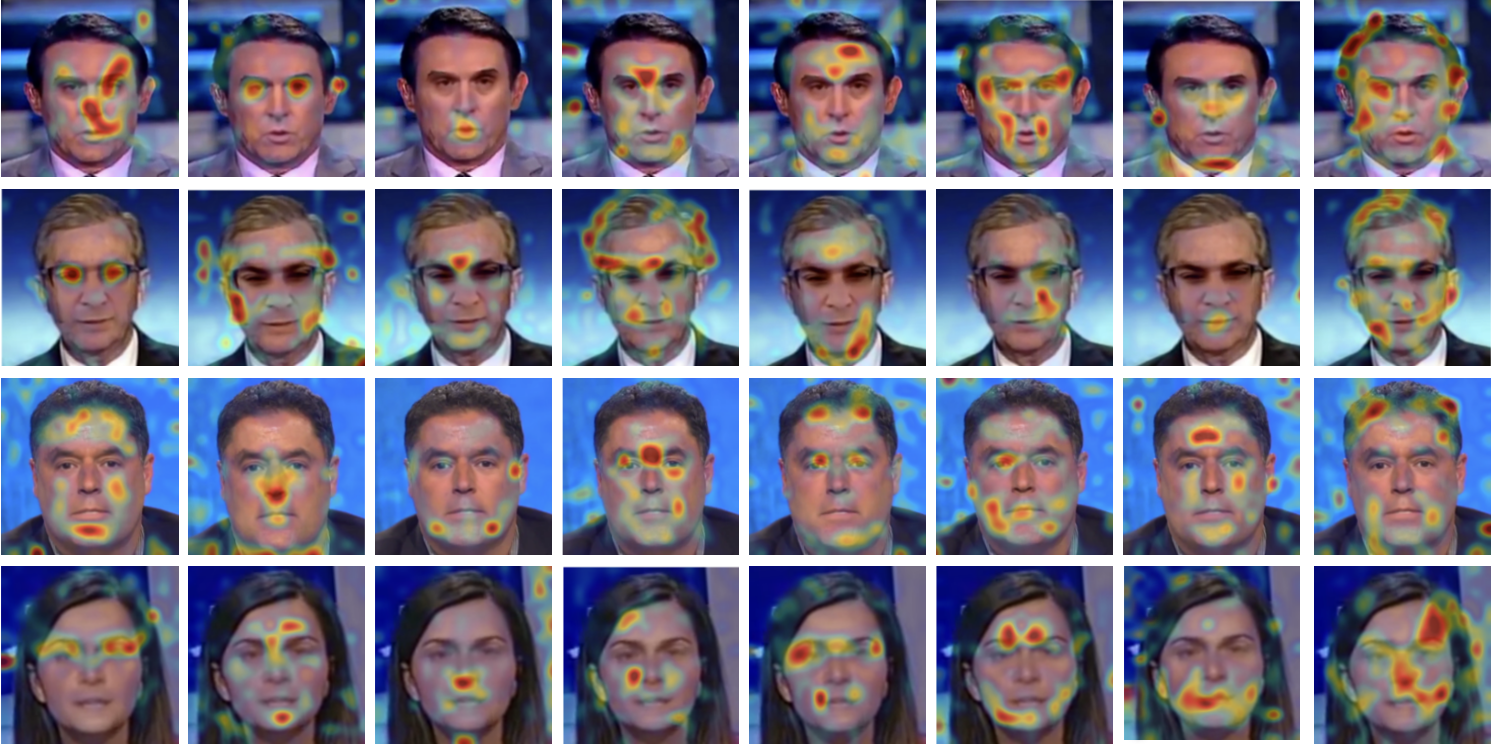}
    \caption{Visualization of spatial prototype activation maps. We visualize the bilinearly upsampled token-to-prototype soft assignments for 8 selected prototypes (columns) across various input faces (rows). Notably, each prototype isolates distinct and consistent localized regions (\textit{e.g.}, eyes, mouth, or blending boundaries) regardless of the subject's identity. This demonstrates that the HSP framework successfully disentangles forgery artifacts into diverse, independent structural patterns, mitigating over-reliance on localized shortcuts.}
    \label{fig:exp-activation-map}
\end{figure*}

\subsection{Visualization of Spatial Prototype Activation Maps.}
To understand the representations learned by HSP, we visualize the spatial activation maps of optimized prototypes. Specifically, we extract token-to-prototype assignment maps from the second HSP module, select 8 representative prototypes, and upsample their assignment weights to the original resolution. As shown in Figure~\ref{fig:exp-activation-map}, different prototypes exhibit distinct spatial preferences across images, with some focusing on localized artifacts around the eyes or mouth and others capturing global facial structures or blending boundaries. This spatial disentanglement validates that orthogonality and balance constraints encourage the decomposition of complex forgery evidence into diverse and independent structural patterns.

\section{Limitation}
Our method improves cross-domain generalization by using learnable spatial prototypes to model forgery representations through multiple latent factors. However, this factorization is driven by representation-level constraints rather than explicit supervision for specific forensic cues. Consequently, the learned prototype factors are not guaranteed to align with physically interpretable manipulation traces, such as blending boundaries, frequency inconsistencies, or rendering artifacts. Future work could introduce finer-grained supervision over different dimensions of the prototype space, for example, through manipulation-specific auxiliary objectives or forgery-region localization~\cite{lav,face-x-ray}, to improve the interpretability and semantic alignment of the learned factors.


\end{document}